\documentclass[reprint,amsmath,amssymb,aps,onecolumn,notitlepage]{revtex4-1}

\RequirePackage[T1]{fontenc}

\RequirePackage{engrec}

\RequirePackage{apptools}	

\usepackage{graphicx} 
\usepackage{amsthm}    
\usepackage{amsmath} 
\usepackage{amssymb} 
\usepackage{array}
\usepackage{tabularx}

\usepackage[all]{xy}

\usepackage{color} 

\usepackage{mathrsfs} 

\usepackage{cancel}

\usepackage{comment}

\definecolor{David}{rgb}{0,.8,.5}

\definecolor{Julien}{rgb}{1,.4,.3}

\newtheoremstyle{break}
{\topsep}{\topsep}%
{}{}%
{\bfseries}{}%
{\newline}{}%

\theoremstyle{break}

\usepackage{graphicx}
\usepackage{dcolumn}
\usepackage{bm}
\usepackage{xcolor}
\usepackage{amsmath}
\usepackage{comment} 
\usepackage{amssymb}
\usepackage{bbold}
\usepackage{amsthm}

\theoremstyle{plain}
\newtheorem{theorem}{Theorem}[section]
\newtheorem{lemma}[theorem]{Lemma}

\newtheorem{corollary}[theorem]{Corollary}
\newtheorem{definition}[theorem]{Definition}

\begin{document}

\title{Gauges, Loops, and Polynomials for Partition Functions of Graphical Models}

\author{Michael Chertkov $^{(1,2)}$, Vladimir Chernyak$^{(3)}$, and Yury Maximov$^{(4,2)}$\\
$^{(1)}$ Program in Applied Mathematics, University of Arizona, Tucson, AZ, USA\\
$^{(2)}$ Skolkovo Institute of Technology, Moscow, Russia\\
$^{(3)}$ Department of Chemistry, Wayne State University, MI, USA\\
$^{(4)}$ T-Division, Los Alamos National Laboratory, Los Alamos, NM, USA}
\email{chertkov@arizona.edu, chernyak@chem.wayne.edu, yury@lanl.gov}

\date{\today}

\begin{abstract}
Graphical models represent multivariate and generally not normalized probability distributions. Computing the normalization factor, called the partition function, is the main inference challenge relevant to multiple statistical and optimization applications.  The problem is $\#$P-hard that is  of an exponential complexity with respect to the number of variables. In this manuscript, aimed at approximating the partition function, we consider Multi-Graph Models  where binary variables and multivariable factors are associated with edges and nodes, respectively, of an undirected multi-graph. We suggest a new methodology for analysis and computations that combines the Gauge Function  technique from \cite{06CCa,06CCb} with the technique developed in \cite{11Gur} and \cite{17AG,17SVa} based on the recent progress in the field of real stable polynomials. We show that the Gauge Function, representing a single-out term in a finite sum expression for the Partition Function which achieves extremum at the so-called Belief-Propagation  gauge, has a natural polynomial representation in terms of gauges/variables associated with edges of the multi-graph. Moreover, Gauge Function can be used to recover the Partition Function through a sequence of transformations allowing appealing algebraic and graphical interpretations. Algebraically, one step in the sequence consists in the application of a differential operator over gauges associated with an edge. Graphically, the sequence is interpreted as a repetitive elimination/contraction of edges resulting in Multi-Graph Models on decreasing in size (number of edges) graphs with the same Partition Function as in the original Multi-Graph Model. Even though the complexity of computing factors in the sequence of the derived Multi-Graph Models and respective Gauge Functions grow exponentially with the number of eliminated edges, polynomials associated with the new factors remain \emph{Bi-Stable} if the original factors have this property. Moreover, we show that BP estimations in the sequence do not decrease, each low-bounding the Partition Function. 
\end{abstract}

\maketitle

\section{Introduction}

Graphical models (GM) are ubiquitous in natural and engineering sciences where one needs to represent a multivariate distribution function with a structure that is expressed in terms of graphical, statistical or deterministic, relations between the variables \cite{63Gal,88Pea,02Mac,08RU,08WJ,09KF,09MM,11MM}. Focusing on the so-called Normal Factor Graph  representation \cite{01For},  where binary variables and factors that express relations between the variables are associated with edges and nodes of the graph, respectively, we are interested in resolving the problem of statistical inference, which entails computing the weighted sum over allowed states. Exact evaluation of the sum, called the Partition Function, is known to be $\#$-P-hard \cite{79Val,86JVV,03Vaz}, that is, of complexity which  likely requires an exponential number of steps.  Subsequently, deterministic and stochastic approximations were made. In this manuscript, we concentrate primarily on the former. (Stochastic methods for the Partition Function estimations are reviewed in \cite{97JS,03Vaz}. See also some related discussions in \cite{18Kol} and below.)

The inference problem can also be stated as an optimization. The variational approach to Partition Function computation dates back to Gibbs \cite{gibbs_2010},  and possibly earlier.  Similar considerations are known in statistics under the name of Kullback--Leibler divergence \cite{kullback1951}. The resulting optimization stated in terms of beliefs (i.e., proxies for probabilities of states) is convex but not tractable because of the exponential number of states (and respectively beliefs). Developing relaxations, and more generally approximations,  for the Gibbs--Kullback--Leibler variational formulation is the primary research to which this manuscript is contributing. 

Theoretical efforts in the field of deterministic estimations of Partition Functions have focused on devising (a) lower and/or upper bounds for GMs of a special type and (b) Fully Polynomial Deterministic Algorithmic Schemes  for even more restrictive classes of GMs. (See Section \ref{sec:bistab-monotone} for an extensive discussion of the low bounds and related subjects. Section \ref{sec:conclusion} for a brief discussion on unification of these ideas with Fully Polynomial Deterministic Algorithmic Schemes.) 

Provable lower bounds for Partition Functions are known for Perfect Matching  problems over bi-partite graphs \cite{11Gur,14GS}, independent set problems \cite{06Wei,11CCGSS,14SS}, and Ising models of attractive (log-supermodular) \cite{07SWW,12Ruo,17LSS} and general \cite{16Ris} types. In a few cases where the exact computation of Partition Functions is polynomial, noticeably GM over tree graphs \cite{35Bet,36Pei,63Gal,88Pea,01For} and also cases where Partition Function becomes a determinant of a polynomial (in the size of the original GM) matrix correspondent to Ising, PM and other specialized models over planar graphs \cite{63Kas,66Fis,82Bar,00GLV,08Val,10CC}, would normally be considered good starting points for analysis of lower bounds on the Partition Functions. 

This manuscript contributes to the line of research  with roots in the tree-graph and dynamic programming (DP) methodology and also its extensions to loopy multi-graphs. The subject has a distinguished history in physics \cite{35Bet,36Pei,09MM}, information theory \cite{63Gal,01For,08RU}, artificial intelligence and machine learning \cite{88Pea,02Mac,09KF}, statistics, and computer science \cite{08WJ,11MM}.  It culminates in the so-called Belief Propagation (BP) analysis, theory, and algorithms. (The term BP was coined by Pearl, who has pioneered related applications in artificial intelligence and machine learning \cite{88Pea}.)  Applied to graphs with loops, as first done by Gallager in the context of the Low-Density Parity Check codes \cite{63Gal},  BP is a practically successful heuristic algorithm, generally lacking quality assurance. The iterative/algorithmic part of BP was connected to the variational Gibbs--Kullback--Leibler interpretation in \cite{05YFW}, where it was shown that (in the case of convergence) the BP algorithm corresponds to a fixed point of the so-called Bethe Free Energy, stated in terms of the marginal beliefs associated with nodes and edges of the GM. (See also Section \ref{subsec:vbp} for details.) In the following, we refer to BP as a fixed point (possibly one of many) of the Bethe Free Energy, assuming that it can be found efficiently \cite{14Shi}. We will also generalize the notion of fixed points to the cases when the minimum of the Bethe Free Energy may be  achieved at a plaquet/side of the belief polytope over which the Bethe Free Energy is defined (not necessarily within the interior of the polytope). 

Heuristic success as well as results claiming exactness of BP for some special optimization problems over loopy graphs (e.g., finding maximum weight perfect matching over bi-partite graphs \cite{08BSS}) have stimulated the design of a number of methods relating results of BP to exact results. These methods include Gauge Transformation  and Loop Calculus of \cite{06CCa,06CCb}, the spanning tree approach of \cite{06Wei}, the cumulant expansion approach of \cite{12WGI}, the graph cover approach of \cite{13Von}, and most recently the Real Stable Polynomial approach of \cite{11Gur,17AG,17SVa}. The first and last approaches are most relevant to this manuscript.

The Gauge Transformation--Loop Calculus method of \cite{06CCa,06CCb} suggests an exact construct exploring invariance of the Partition Function with respect to special transformations of factors, called gauges, also related to the so-called re-parametrizations of \cite{03WJW} and holographic transformations of \cite{08Val}. It was shown that BP corresponds to a special choice of gauges, which then lead to expressing the Partition Function in terms of the so-called generalized loops, where each generalized loop contribution is stated explicitly in terms of the underlying BP solution. The Gauge Transformation--Loop Calculus approach was utilized (1) to prove that BP provides a lower bound for attractive Ising models with some additional technical constraints in \cite{08WSW} (it was then shown in \cite{12Ruo,17Ruo} through the use of the graph cover approach of \cite{13Von} that the additional constraints are insignificant); (2) to prove that BP is exact asymptotically for an ensemble of independent set problems \cite{11CCGSS}; (3) to relate matching models, Fermion models of statistical physics with loop and determinant considerations \cite{08CCa,08CCb}; (4) to approximate Partition Function in planar GM \cite{08CCT,10GKC}; (5) to apply Loop Calculus to permanent (Partition Function of perfect matching model over bi-partite graph) \cite{10WC}, to provide a proof that is alternative to the original \cite{11Gur} for the fact that BP results in a lower bound for permanent, and then construct a sequence of fractional-BP approximations for permanents \cite{13CY}; (6) to build a Fully Polynomial Randomized Approximation Schemes for a subclass of planar GMs \cite{16ASS} by sampling Loop Series; and (7) to construct a provable lower bound for Partition Function in the case when BP fails to provide such guarantees \cite{17ASS} by finding an optimal non-BP gauge that certifies that all terms in the Loop Series are positive. 

The Real Stable Polynomials approach to the Partition Function, first developed for permanents in \cite{11Gur} and then generalized to binary GM over (normal) bi-partite graphs with submodular factors in \cite{17AG,17SVa}, is built on the recent progress in the Real Stable Polynomials theory \cite{06Bra,07BBL,08BB,09BB}. The essence of the approach  is in representing the Bethe Free Energy as a polynomial optimization and then showing that the Partition Function is a result of a sequential application of edge-local differential operators to the Bethe Free Energy/BP estimate of Partition Function. It was shown in \cite{17SVa} that if all polynomials associated with nodes of the GM are Real Stable  and the graph is bi-partite, then each application of the edge-local differential operator ensures that the respective Partition Function estimate does not decrease, thus resulting in the statement that BP (the zero term in the sequence) provides a lower bound for the partition function (the last term in the sequence). In a related paper, \cite{17AG} a polynomial version of the GM statement of \cite{17SVa} was proven for a more general case of the so-called Bi-Stable polynomials over arbitrary graphs.

\subsection{Contributions of this manuscript}

\begin{figure}
	\centering 
	\includegraphics[width=6in,page=6]{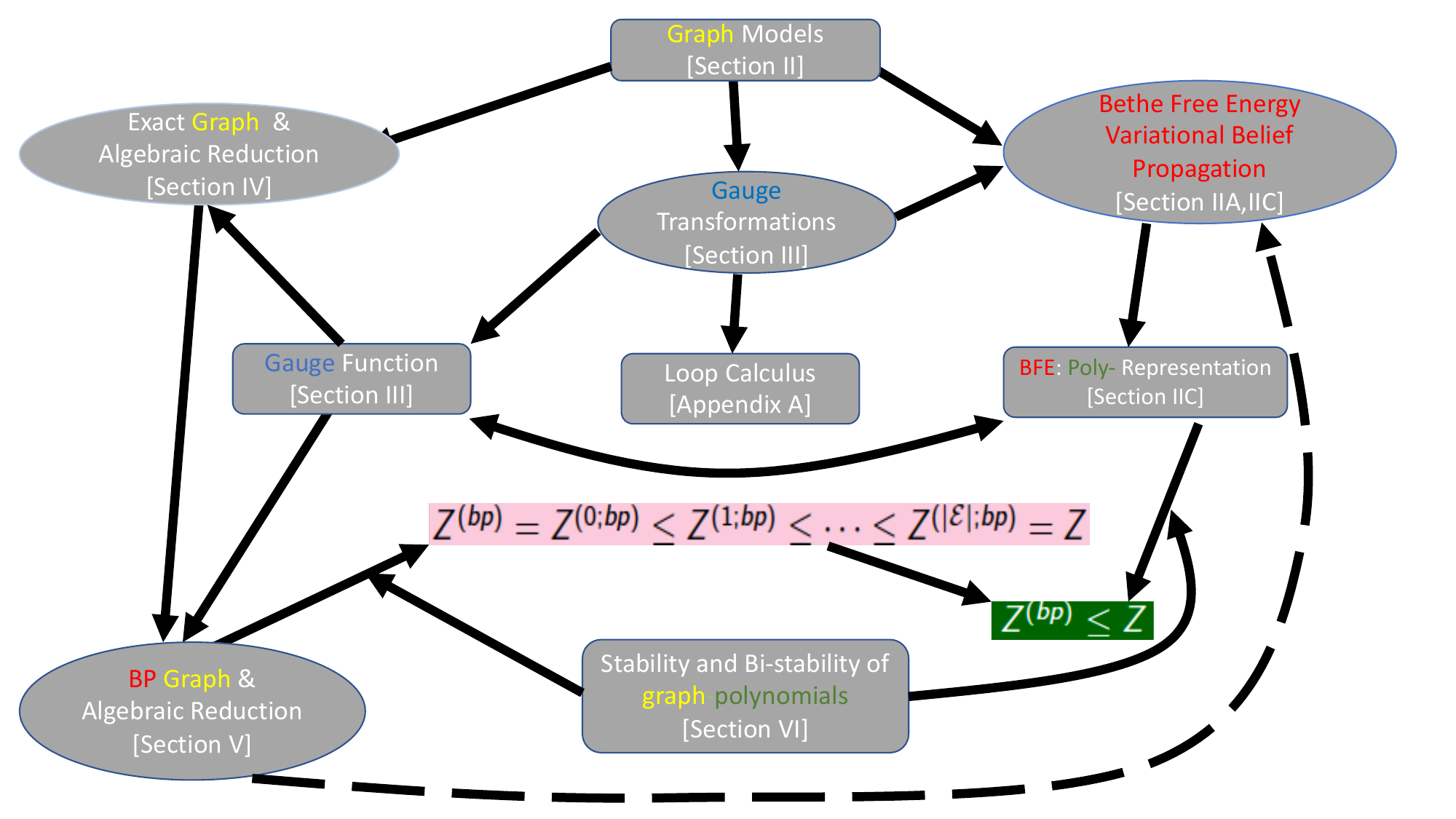}
	\caption{Relation diagram of steps and concepts of the manuscript with links to respective Sections. \label{fig:scheme}}
\end{figure}

We complement the Real Stable Polynomials approach of \cite{11Gur,17AG,17SVa} by merging it with the Gauge Transformation approach of \cite{06CCa,06CCb}, and thus in a sense generalize both. Our approach consists of the following steps (see also Fig.~(\ref{fig:scheme}) presenting a diagram of relations between the manuscript's steps and concepts) :
\begin{itemize}
    \item Generalize Variational Bethe Free Energy approach (from the case of normal graphical models) to the Multi-Graphical Models. Show that solution of any Soft Multi-Graph Model is attained strictly within a polytope of parameters -- so-called Belief Polytope -- describing the solution. All further results reported in the paper (unless specified otherwise) apply strictly speaking only to Multi-Graph Models (Multi-GM) which are soft -- even though some of the factors may be infinitesimally small. (Section \ref{sec:pre}).

    \item Restate the Gauge Transformation expression for Partition Function from \cite{06CCa,06CCb} as a series of polynomials in variables/gauges. Single out a term from the series, which we call the Gauge Function, relate stationary points of the Gauge Function, so-called BP-gauges and show that the minimum of the Bethe Free Energy is achieved at the maximal BP-gauge. (Section \ref{sec:Gauge Transformation-BP}).
    
    \item Introduce a sequence of Multi-GMs, where each new member is a result of an edge contraction (graphically) or summation over respective edge variables (algebraically). Build BP polynomial (principal polynomial evaluated at the optimal BP gauge) for each Multi-GM in the sequence such that the last term is the Partition Function (constant) corresponding to the fully contracted graph. Introduce BP-optimal gauge for each Multi-GM and show that BP-optimal estimation stays exact in the process of contraction of  a normal edge, however, it becomes approximate respective contraction of a self-edge.  (Section \ref{sec:elim}.)
    
    \item Observe that the Bi-Stability of polynomials correspondent to factors of the original Multi-GM results in the Real Stable Polynomials of each factor in each Multi-GM of the aforementioned contraction sequence. Show that the variational BP solution (correspondent to the minimum of the respective Bethe Free Energy) for each next Multi-GM in the sequence is larger or equal to BP if all factors in the original Multi-GM correspond to Real Stable Polynomials polynomials.
    A direct corollary of this construction is the desired statement that the BP optimal estimation for the Partition Function of the original Multi-GM low bounds the exact Partition Function. 
    (Section \ref{sec:mix_der}.)
\end{itemize}

We present in Section \ref{sec:pre}, for the purpose of setting terminology and self-consistency of the presentation, introductory material for the Bethe Free Energy approach. For the Loop Calculus approach of \cite{06CCa,06CCb}, we present introductory material in Appendix \ref{app:LS}. Section \ref{sec:conclusion} is reserved for discussions of the results and the path forward.

\section{Preliminaries}
\label{sec:pre}

\begin{figure}
	\centering 
	\includegraphics[width=2.3in,page=1]{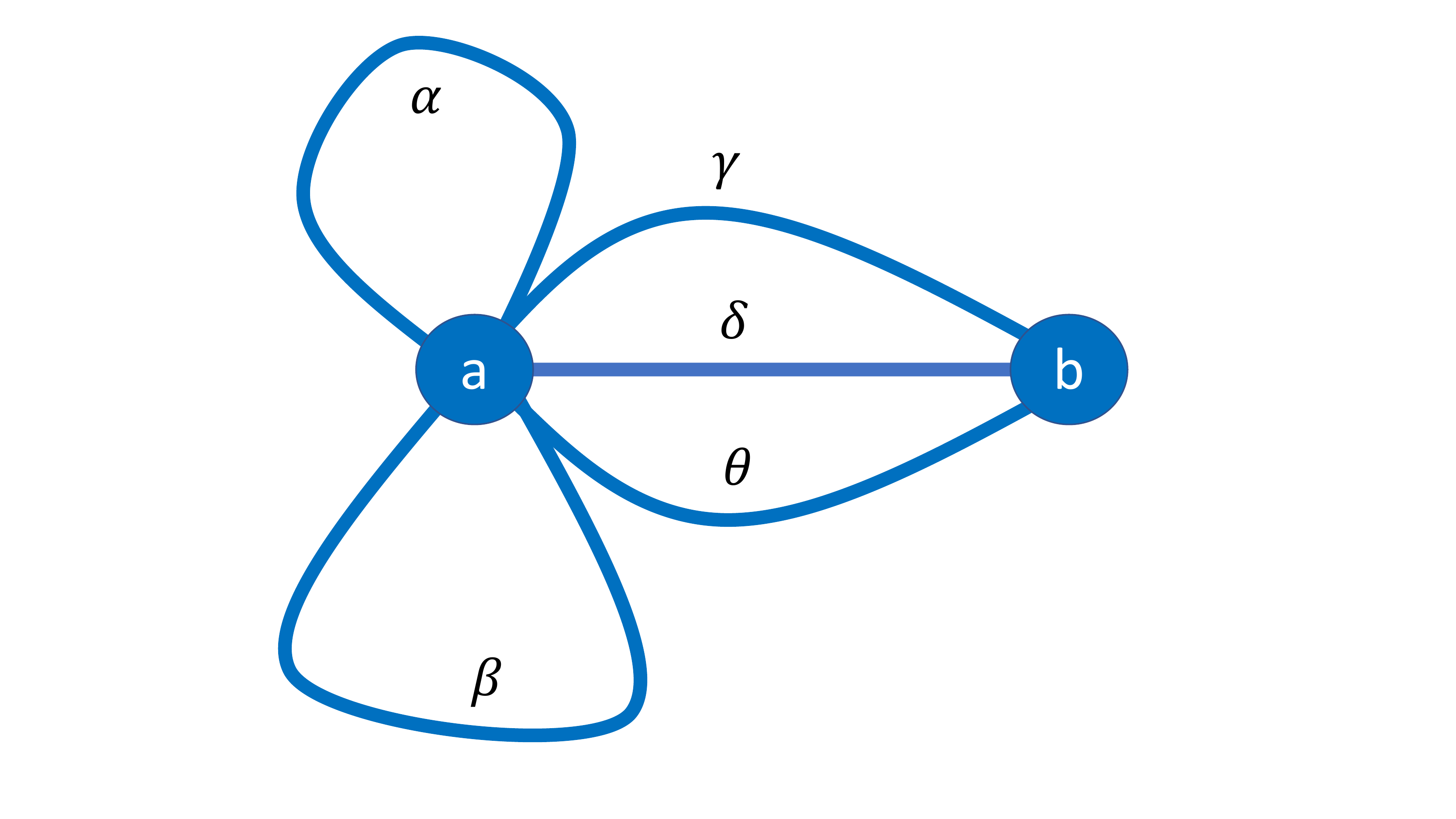}
	\includegraphics[width=2.3in,page=2]{figs/MF-GM.pdf}
	\includegraphics[width=2.3in,page=3]{figs/MF-GM.pdf}
	\caption{Multi-Factor Graphical Model notations for undirected (left sub-figure) and directed (middle and right sub-figures) edges are illustrated. ${\cal V}=\{a,b\}$, ${\cal E}=\{\alpha,\beta,\gamma,\delta,\theta\}$ and ${\cal E}_d=\{\alpha_d,\bar{\alpha}_d,\beta_d,\bar{\beta}_d,\gamma_d,\bar{\gamma}_d,\delta_d,\bar{\delta}_d,\theta_d,\bar{\theta}_d\}$ are the sets of nodes, set of undirected edges and set of directed edges, respectively, where 
	$\bar{\cdots}$ is the notation used to denote the directed edge $\cdots$ reversal, thus $\bar{\bar{\alpha}}_d=\alpha_d$. $e(a)=\{\alpha,\beta,\gamma,\delta,\theta\}$, $e(b)=\{\gamma,\delta,\theta\}$  and $e_d(a)=\{\alpha_d,\bar{\alpha}_d,\beta_d,\bar{\beta}_d,\bar{\gamma}_d,\delta_d,\theta_d\}$, $e_d(b)=\{\gamma_d,\bar{\delta}_d,\bar{\theta}_d\}$ describe functions, $e(\cdots):{\cal V}\to {\cal E}$ and $e_d(\cdots):{\cal V}\to {\cal E}_d$, mapping a node to the set of directed and undirected edges, respectively, of the multi-graph shown in the figure. $u(e):{\cal E}_d\to {\cal E}$ describes function mapping directed edges into respective undirected edges; $u(\alpha_d)=u(\bar{\alpha}_d)=\alpha$. $v(\alpha):{\cal E}\to{\cal V}^2$ describes function mapping an undirected edge to its end nodes. \label{fig:MF-GM} }
\end{figure}

We consider multi-graph generalization of \cite{06CCa,06CCb}. Following terminology of \cite{04Loe}, we may also call it binary Factor-Multi-GM: variables are associated with edges, and factor functions (or simply factors) are associated with nodes of the undirected multi-graph, ${\cal G}\doteq ({\cal V},{\cal E})$, thus allowing multi-edges between two different nodes and multiple self-edges associated with a single node, where ${\cal V}$ and ${\cal E}$ are the sets of nodes and edges, respectively. The main reason for dealing with multi-graphs is that certain geometrical graph transformations, namely edge contraction, introduced in Section~\ref{sec:bistab-monotone} turn simple graphs (no multiple edges, no self-edges) into multi-graphs.  

It is also useful to introduce an oriented version of the undirected multi-graph, i.e., a multi-graph, equipped with orientation. We will then denote ${\cal E}_d$ the set of directed edges including the original orientation and its reverse.  $e_d(a)$, $v(\alpha)$ and $u(\alpha_d)$ will denote, respectively,  the set of directed edges associated with the node $a$, two nodes associated with the undirected edge, $\alpha$, and undirected edge, $\alpha$, associated with the directed edge $\alpha_d$. Also, and abusing notations a bit, $\alpha_d$ may denote the primary oriented edge for previously introduced undirected edge $\alpha$; $\alpha_d\in e_d(a)$ denotes a variable directed edge picked from the $e_d(a)$ set. See Fig.~\ref{fig:MF-GM} for clarifying example.

\begin{definition}[(Multi-) Graph Model] 
\label{def:GM}
Multi-GM describes factorization for the probability of a binary-component vector, 
$\sigma\doteq (\sigma_\alpha=0,1|\alpha\in{\cal E})\in\{0,1\}^{|{\cal E}|}$,
consistent with the (multi)-graph:
\begin{eqnarray}
p(\sigma)\doteq\frac{f(\sigma)}{Z},\quad f(\sigma)\doteq
\prod_{a\in{\cal V}} f_a(\sigma_a),\quad Z\doteq \sum_\sigma  f(\sigma).
\label{eq:GM}
\end{eqnarray}
Here $\sigma_a$ is a sub-vector of $\sigma$ built from all components of the latter containing node $a$, i.e. $\sigma_a \in \Sigma_{a}\doteq \{0, 1\}^{e_d(a)}$.  
\end{definition}

Notice that the Partition Function, $Z$, defined in Eq.~(\ref{eq:GM}) as a summation over all configurations, $\sigma$, allows a recast in terms of the following exact variational principle.
\begin{theorem}[Gibbs-Kullback-Leibler (Gibbs--Kullback--Leibler) Variational Reformulation for Partition Function, in the spirit of \cite{gibbs_2010,kullback1951}]
\label{theorem:Gibbs--Kullback--Leibler}
The Partition Function, $Z$, defined in Eq.~(\ref{eq:GM}) can be computed through the following optimization 
\begin{eqnarray}
-\log Z=\min_{\mathbb{b}}\left. \sum_{\sigma\in\{0,1\}^{|{\cal E}|}} \mathbb{b}(\sigma)\log \frac{\mathbb{b}(\sigma)}{\prod\limits_{a\in{\cal V}} f_a(\sigma_a)}\right|_{\begin{array}{l} \forall \sigma:\quad \mathbb{b}(\sigma)\geq 0\\ \sum_\sigma \mathbb{b}(\sigma)=1\end{array}},\label{eq:var-exact}
\end{eqnarray}
where $\mathbb{b}=(\mathbb{b}(\sigma)|\forall \sigma\in \sigma\in\{0,1\}^{|{\cal E}|} )$, and $\mathbb{b}(\sigma)$ are  beliefs (i.e., proxies for probabilities) of the state $\sigma$. 
\end{theorem}

The optimization (\ref{eq:var-exact}) is convex but not practical because the number of states (and number of respective beliefs) is exponential in the system size (number of edges).

\begin{theorem}[Exact Maximum-A-Posteriori as a Linear Programming]
\label{theorem:LP-ML}
The Maximum-A-Posteriori versions of Eqs.~(\ref{eq:GM},\ref{eq:var-exact}) are
\begin{eqnarray}
E\doteq -\min\limits_{\sigma} \log f(\sigma)=-\left.\min_{\mathbb{b}} \sum_\sigma \mathbb{b}(\sigma) \sum_{a\in{\cal V}}\log f_a(\sigma_a)\right|_{\begin{array}{l} \forall \sigma:\quad \mathbb{b}(\sigma)\geq 0\\ \sum_\sigma \mathbb{b}(\sigma)=1\end{array}}.
\label{eq:ML}
\end{eqnarray}
\end{theorem}
Notice that the formulation on the right of Eq.~(\ref{eq:ML}) is an LP over (exponentially many) belief variables.

\subsection{Variational Belief Propagation}
\label{subsec:vbp}

Belief propagation (BP) is a popular and practical tool that approximates original beliefs via marginal beliefs according to the following dynamic programming (DP) expression \cite{08RU,02Mac,08WJ,09MM,11MM}, which gets the following form when stated for the Multi-GM model (\ref{eq:GM})
\begin{eqnarray} \label{eq:b-BP}
 \mathbb{b}(\sigma)\approx \frac{\prod\limits_a b_a(\sigma_a)}
{\prod\limits_{\alpha\in{\cal E}}\beta_\alpha(1-\beta_\alpha)}
,\mbox{ s.t. } \left\{\begin{array}{c}
 \forall a\in{\cal V},\ \forall \sigma_a\in \Sigma_{a}^{(0)}:\ b_a(\sigma_a)\doteq\sum\limits_{\sigma\setminus \sigma_a} \mathbb{b}(\sigma);\\ \forall \alpha\in{\cal E}:\ \beta_\alpha=\sum\limits_{\sigma\setminus\sigma_\alpha}^{\sigma_\alpha=1}\mathbb{b}(\sigma);\end{array}\right.
\end{eqnarray}
which would be exact for a tree-graph (no loops in GM).  Reducing description from the exponential in size vector of original beliefs to the linear in the size of the graph (assuming that node degree in the Multi-GM is $O(1)$) vector of marginal beliefs
\begin{eqnarray}
b=\left.\left(b_a(\sigma_a)\right| \forall a\in{\cal V},\ \sigma_a \in \Sigma_{a}\right),\quad 
\beta=\left.\left(\beta_\alpha\right| \forall \alpha\in{\cal E}\right).
\label{eq:beta+b}
\end{eqnarray}
considered over the following marginal polytope 
\begin{eqnarray}
 \Pi\doteq  \left\{ (b, \beta) \in [0,1]^{\Sigma} \times [0, 1]^{\cal E}  \left| \begin{array}{c}
 \forall a \in {\cal V}:\  \sum_{s \in \Sigma_{a}} b_{a}(s) = 1;\\ \forall a \in {\cal V},\ \forall \alpha_d \in {\cal E}_{\rm d}(a):\quad \sum_{s \in \Sigma_{a}}^{s_{u(\alpha_d)}=1} b_{a}(s) = \beta_{u(\alpha_d)}\end{array} \right. \right\}.
\label{eq:Pi}
\end{eqnarray}
and substituting Eq.~(\ref{eq:b-BP}) into Eq.~(\ref{eq:var-exact}), one arrives at the following optimization.
\begin{definition}[Bethe Free Energy and Variational Belief Propagation approximation for the Partition Function, by analogy with \cite{05YFW}]
\label{def:BFE}
Variational BP estimation for the Partition Function, $Z^{{\rm (vbp})}$,  is defined according to 
\begin{eqnarray}
&& F^{(\rm {vbp})}\doteq -\log Z^{{\rm (vbp})}\doteq \min_{(b,\beta)\in \Pi} F^{(bp)}(b,\beta),\label{eq:vbp}\\
&& F^{(bp)}(b,\beta)\doteq E^{(bp)}(b)- S^{(bp)}(b,\beta),
\label{eq:Fbp}\\
&& E^{(bp)}(b)\doteq -\sum_{a\in{\cal V}_a} \sum_{\sigma_a\in\Sigma_a} b_a(\sigma_a)\log f_a(\sigma_a),
\label{eq:Ebp}\\
&& S^{(bp)}(b,\beta)\doteq \sum_{a\in{\cal V}_a} \sum_{\sigma_a\in\Sigma_a} b_a(\sigma_a)\log b_a(\sigma_a)-
\sum_{\alpha\in{\cal E}} \left(\beta_\alpha\log\beta_\alpha+(1-\beta_\alpha)\log(1-\beta_\alpha)\right),
\label{eq:Sbp}
\end{eqnarray}
where $F^{(\rm {vbp})}$ is the Variational Bethe Free Energy, $F^{(bp)}(b,\beta)$, $E^{(bp)}(b)$ and $S^{(bp)}(b,\beta)$ are the Bethe Free Energy, Bethe Self Energy and Bethe Entropy  functions of marginal beliefs.
\end{definition}

We further notice that Eq.~(\ref{eq:vbp}) may be restated as a polynomial optimization \cite{17SVa}. To derive the polynomial representation one, first, rewrites Eq.~(\ref{eq:vbp}) as
\begin{eqnarray}
&& \log Z^{({\rm vbp})}=\max_{\beta\in[0,1]^{{\cal E}}} \left. \left(S^{(bp-r)}(\beta)-E^{(bp-r)}(\beta)\right)\right., \label{eq:vbp-r}\\
&& S^{(bp-r)}(\beta)\doteq \sum_{\alpha\in {\cal E}}\left(\beta_{\alpha}\log\beta_{\alpha}+(1-\beta_{\alpha})\log(1-\beta_{\alpha})\right),\label{eq:Sbp-r}\\
&& E^{(bp-r)}(\beta)\doteq \min\limits_{b\in \Pi_r(\beta)} \sum_{a \in {\cal V}} \sum_{\sigma_a \in \Sigma_{a}} b_a(\sigma_a)\log \frac{b_a(\sigma_a)}{f_a(\sigma_a)}\label{eq:Ebp-R}\\
&&  \Pi_r(\beta)\doteq  \left\{ b \in [0,1]^{\Sigma}  \left| \begin{array}{c}
\forall a \in {\cal V}:\  \sum_{s \in \Sigma_{a}} b_{a}(s) = 1;\\ \forall a \in {\cal V},\ \forall \alpha_d \in e_d(a):\quad \sum_{s \in \Sigma_{a}}^{s_{u(\alpha_d)}=1} b_{a}(s) = \beta_{u(\alpha_d)}\end{array} \right. \right\}.
\label{eq:Pi-r}
\end{eqnarray}
where $S^{(bp-r)}(\beta)$ and $E^{(bp-r)}(\beta)$ are the reduced BP entropy and the reduced BP self-energy functions, respectively, dependent only on the vector of edge probabilities, $\beta$.  Applying strong duality to the reduced BP self-energy  one derives
\begin{eqnarray}
&& E^{(bp-r)}(\beta)=\sup_{x \in \mathbb{R}_{+}^{{\cal E}_{\rm d}}}\left(\sum_{\alpha\in{\cal E}}\beta_{\alpha}\log (x_{\alpha_d} x_{\bar{\alpha}_d})
-\sum_{a\in{\cal V}} \log(h_a(x_a))\right), 
\label{eq:Ebp-r1}\\
&& h_a(x_a)\doteq \sum_{s \in \Sigma_{a}}f_a(s) \prod_{\alpha \in {\cal E}_{\rm d}(a)} x_{\alpha}^{s_{\alpha}}
, \label{eq:h_a}
\end{eqnarray}
where, $x \doteq (x_{\alpha}>0 | \forall \alpha \in {\cal E}_{\rm d}) \in \mathbb{R}_{+}^{{\cal E}_{\rm d}}$ and
$x_a \doteq(x_{\alpha}>0|\forall \alpha\in e_{\rm d}(a)) \in \mathbb{R}_{+}^{e_{\rm d}(a)}$. The $\log(x_{\alpha})$ components, with $\alpha \in {\cal E}_{\rm d}$ were introduced as Lagrangian multipliers (dual variables) for the belief consistency conditions in Eq.~(\ref{eq:Ebp}). Combining Eqs.~(\ref{eq:vbp},\ref{eq:Fbp},\ref{eq:vbp-r},\ref{eq:Sbp-r},\ref{eq:Ebp-r1}), one arrives at the following statement.
\begin{theorem}[Polynomial Max-Min Representation for Variational BP, multi-graph version of Theorem 3.1 of \cite{17SVa}]
\label{theorem:max-min_VBP}
Variational BP, defined  Eqs.~(\ref{eq:vbp},\ref{eq:Fbp},\ref{eq:Ebp},\ref{eq:Sbp}), can also be stated as the following max-min optimization:
\begin{eqnarray}
&& Z^{({\rm vbp})} = \sup\limits_{\beta\in[0;1]^{\cal E}} \min\limits_{x \in \mathbb{R}_{+}^{{\cal E}_{\rm d}}} {\cal L}(\beta,x), \label{eq:vbp-p}\\ && {\cal L}(\beta,x)\doteq 
\left(\prod_{\alpha \in{\cal E}}\beta_{\alpha}^{\beta_{\alpha}} (1-\beta_{\alpha})^{1-\beta_{\alpha}}\right) \prod_{a\in{\cal V}}\frac{h_a(x_a)}{\prod\limits_{\alpha_d\in {\cal E}_d(a)}x_{\alpha}^{\beta_{u(\alpha_d)}}},\label{eq:L}
\end{eqnarray}
\end{theorem}

\subsection{Linear Programming  Relaxation for Maximum-A-Posteriori}
\label{subsec:LP-BP}

\begin{definition}[Linear programming--Belief Propagation approximation]
Linear Programming -- Belief Propagation approximation for Maximum-A-Posteriori optimization (\ref{eq:ML}) is 
\begin{eqnarray}
E^{({\rm lp-bp})}\doteq -\min_{(b,\beta)\in P_B} \left(\sum_{a \in {\cal V}} \sum_{\sigma_a \in \Sigma_{a}} b_a(\sigma_a)\log f_a(\sigma_a)\right).\label{eq:lp-bp}
\end{eqnarray}
\end{definition}
Notice  that Linear Programming -- Bellief Propagation is tractable, and it can be considered both as the ``entropy-free'' version of the optimization (\ref{eq:vbp}) and also as a relaxation of the exact LP formulation (\ref{eq:ML}) and therefore results in the following statement.
\begin{theorem}[Lower bounding by Linear Programming -- Belief Progagation (see for example \cite{08Joh,08WJ,10Son} and references therein)]
\label{theorem:LP-BP}
Linear Programming -- Belief Progagation lower bounds exact self-energy.
$E^{({\rm lp-bp})}\leq  E$.
\end{theorem}
Note that the same Linear Programming -- Belief Progagation is known under the name of  ``Basic Linear Programming Relaxation" in the community analyzing Constrain Satisfaction Problems.  See, e.g. \cite{15KTZ} and references therein. 

\subsection{Variational Belief Propagation in the Soft Model}

\begin{definition}[Soft Multi-GM]
	\label{def:MGM-soft}
	If $\forall a\in{\cal V},\ \forall\sigma_a\in\Sigma_a:\ f_a(\sigma_a)>0$, the Multi-GM  is called soft.
\end{definition}

\begin{theorem}[Variational BP of Soft Multi-GM -- in the spirit of Proposition 6 of \cite{05YFW}]
	\label{theorem:Soft-VBP}
	Minimum in Eq.~(\ref{eq:vbp}) is achieved within the interior of $\Pi$ in the case of soft Multi-GM.
\end{theorem}
\begin{proof}
	The theorem is proved in three steps: first, one shows that the minimum in Eq.~(\ref{eq:vbp}) cannot be achieved at $\beta$ such that at one edge, $\alpha$ (at least one edge), $\beta_\alpha$ is exactly zero or one; (b) given (a) one checks explicitly that when all factors are soft (and thus no terms in the multi-linear polynomials $h_a$ are zero) the minimum over $x$ in Eq.~(\ref{eq:vbp-p}) is achieved at a finite $x\in \mathbb{R}_+^{{\cal E}_d}$; finally, given that the Lagrangian multipliers, $x$,  for the edge and node belief consistency are all finite, the minimum in Eq.~(\ref{eq:vbp}) can only be achieved at $\forall a\in{\cal V},\ \forall\sigma_a\in\Sigma_a:\quad b_a(\sigma_a)\in ]0;1[$. 
	
	Therefore, only the first step is left to be proven. We present here only a sketch of the proof. Assume that $\exists \alpha\in{\cal E}$ such that the minimum in Eq.~(\ref{eq:vbp})  is achieved at  $\beta_\alpha=0$. Our strategy consists in showing that one can find a direction from the point on the polytope boundary towards interior along which $-\log {\cal Z}(b)$ will decrease, thus arriving at a contradiction. Indeed, when $\beta_\alpha=\epsilon>0$ with $\epsilon\to 0$, one derives that according to the belief consistency relations, i.e. equality relations between beliefs embedded  in the definition (\ref{eq:Pi}) of the polytope $\Pi$, all $b_a(\sigma_a)=O(\epsilon)$ where $\forall a\in v(\alpha)$, $\sigma_a$ is consistent with $\sigma_\alpha=1$.  Moreover, one may redistribute the $O(\epsilon)$ perturbations over these  $b_a(\sigma_a)$ such that $\beta_\gamma$, where $\gamma\neq \alpha$, do not depend on the $\epsilon$-perturbation at all. On the other hand,   $\epsilon$-corrections to $-\log {\cal Z}(b)$ are $O(\epsilon\log\epsilon)$. The corrections  originate from the entropy contributions associated with $O(\epsilon)$ beliefs of two types --- associated with $\beta_\alpha$ and associated with the respective $b_a(\sigma_a)$.  Accurate counting of the contributions results in the overall $\epsilon\log\epsilon$ correction to $-\log {\cal Z}(b)$, where $2\epsilon\log\epsilon$ term comes from the two  $b_a(\sigma_a)=O(\epsilon)$ contributions and one $-\epsilon\log\epsilon$ term comes from the single $\beta_\alpha$ contribution. The resulting, $\epsilon\log\epsilon$,  is negative and it decreases with increase in $\epsilon$,  thus leading to the contradiction.  Similar consideration, now with, $\beta_\alpha=1-\epsilon$, where $\epsilon>0$, $\epsilon\to 0$,  results in the statement that at the optimum $\beta_\alpha$ cannot be equal to unity.
\end{proof}

Note that solution of Eq.~(\ref{eq:vbp}) can be on the boundary of the $\Pi$ polytope if the Multi-GM is hard. See \cite{10WC,14Lel} for discussion of special hard cases, e.g. of the perfect matching problem, where solution is achieved at the boundary of $\Pi$.

Theorem \ref{theorem:Soft-VBP} guarantees that an infinitesimally weak softening of a hard model (achieved by adding an infinitesimal positive correction to $f_a(\sigma_a)=0$ factors) shifts a solution of the optimization (\ref{eq:Fbp}) into the interior of the polytope.  We will use this softening feature of Multi-GM later in Section \ref{subsec:BP-G} to relate Variation Belief Propagation formulations and solutions discussed in this Section to the Gauge Transformation and Belief Propagation Equations we are switching out attention to in the next Section.

\section{Gauge Transformation and Belief Propagation Equations}
\label{sec:Gauge Transformation-BP}

\subsection{Gauge Transformation}
\label{subsec:Gauge Transformation}

\begin{definition}[Gauge Transformation, \cite{06CCa,06CCb}]
Gauge Transformation is a multi-linear transformation of the GM factors:
\begin{eqnarray}
&& \forall a \in {\cal V},\quad \forall \sigma_a \in \Sigma_{a}^{(0)}: \quad f_a(\sigma_a) \mapsto \tilde{f}_a(\sigma_a|G) \doteq \sum_{\varsigma_a \in \Sigma_{a}} f_a(\varsigma_a) \prod_{\alpha \in {\cal E}_{\rm d}(a)}G_{\alpha}(\sigma_\alpha,\varsigma_{\alpha})
\label{eq:Gauge Transformation},
\nonumber
\end{eqnarray}
which keeps the Partition Function invariant; that is, 
\begin{eqnarray}
\forall G:\quad && Z = \sum_{\sigma \in S} \prod_{a\in{\cal V}} f_a(\sigma_a) = \sum_{\sigma \in S^{(0)}} \prod_{a\in{\cal V}} \tilde{f}_a(\sigma_a|G)  = \sum_{\sigma \in S^{(0)}} z(\sigma|G), \nonumber \\ &&
z(\sigma|G) \doteq \sum_{\varsigma \in S^{(0)}} \prod_{a \in \cal{V}} f_a(\varsigma_a)
\prod_{\alpha \in {\cal E}_{\rm d}(a)} G_{\alpha}(\sigma_{\alpha},\varsigma_{\alpha}),
\label{eq:Z_inv}
\end{eqnarray}
where  $\varsigma_a \in \Sigma_{a}^{(0)} \doteq \{0, 1\}^{{\cal E}_{\rm d}(a)}$.
\end{definition}
It is straightforward to check that Eq.~(\ref{eq:Z_inv}) holds if the following condition is met.
\begin{theorem}[Orthogonality of Gauge Transformation \cite{06CCa,06CCb}]
Gauge Transformation $2\times 2$ (in the case of a binary alphabet) matrices satisfy
\begin{eqnarray}
\forall \alpha\in{\cal E},\quad  G_{\alpha_d}^T*G_{\bar{\alpha}_d}=\mathbb{1}_{\alpha},
\label{eq:GC}
\end{eqnarray}
where $\alpha_d$ and $\bar{\alpha}_d$ mark two directed siblings of $\alpha$, and the matrices, $G_{\alpha_d}$ and $G_{\bar{\alpha}_d}$ are non-singular with real-valued components. 
\end{theorem}
To lift the gauge-constraint (\ref{eq:GC}), one introduces the following explicit representation for $G$.
\begin{definition}[Polynomial, $x$-, Representation of Gauges]
We call the following representation for $G$, polynomial- or $x$-representation.
\begin{eqnarray}
G = (G_{\alpha} \, | \, \alpha \in {\cal E}) \quad G_{\alpha} = \frac{1}{(x_{\alpha_d}x_{\bar{\alpha}_d})^{1/4} \sqrt{1 + x_{\alpha_d}x_{\bar{\alpha}_d}}} \left(\begin{array}{cc} \sqrt{x_{\bar{\alpha}_d}} & x_{\alpha_d} \sqrt{x_{\bar{\alpha}_d}}\\ -x_{\bar{\alpha}_d} \sqrt{x_{\alpha_d}} & \sqrt{x_{\alpha_d}}\end{array}\right),
\label{eq:G-x}
\end{eqnarray}
where the vector $x$ is positive component-wise, i.e., $x_{\alpha} > 0, \, \forall \alpha \in {\cal E}_{\rm d}$.
\end{definition}
A number of remarks are in order. 
\begin{itemize}
    \item $x$-representation for $G$ (\ref{eq:G-x}) satisfies Eq.~(\ref{eq:GC}) automatically (by construction). 
    \item The trivial, $G  =\mathbb{1}$, case is recovered in the $(x_{\alpha_{d}} = x_{\bar{\alpha}_{d}}) \to 0$ limit. 
    \item Emergence of negative components in the matrix (in the lower left corner of the representation of Eq. (\ref{eq:G-x})) is unavoidable in order to ensure validity of Eq.~(\ref{eq:GC}). 
    \item Parameterized according to Eq.~(\ref{eq:G-x}), $z(\sigma|G)$, defined in Eq.~(\ref{eq:Z_inv}), adopts the following form:
\begin{eqnarray}
z(\sigma|G)=z(\sigma|x)&=&\left(\prod_{\beta\in{\cal E}} \frac{1}{1+x_{\beta_d}x_{\bar{\beta}_d}}\right) \nonumber\\ &\times& \prod_{a \in {\cal V}} \left(\sum_{\varsigma_a \in \Sigma_{a}^{(0)}} \left(\prod_{\alpha \in {\cal E}_{\rm d}(a)} \left(
x_{\alpha}^{\varsigma_{t(\alpha)}} (x_{\alpha} x_{t(\alpha)})^{(1/2-\varsigma_{\alpha})\sigma_{\alpha}}
(-1)^{(1-\varsigma_{\alpha})\sigma_{\alpha}}\right)\right)f_a(\varsigma_a)\right)
\label{eq:Z_sigma-x}\\ 
&=&  \left(\prod_{\alpha\in{\cal E}}\frac{(x_{\alpha_d} x_{\bar{\alpha}_d})^{\sigma_{\alpha_d}}}{1+x_{\alpha_d}x_{\bar{\alpha}_d}}\right) \prod_{a \in {\cal E}} Q_a (x_a;\sigma_a), \textrm{where} \label{eq:Z_sigma-x-2}\\ 
Q_a (x_a;\sigma_a)&\doteq& \sum_{\varsigma \in S^{(0)}}f_a(\varsigma_a)\prod_{\alpha\in {\cal E}_{\rm d}(a)} \left(x_{\alpha}^{\varsigma_{\alpha}}(-1)^{\sigma_{\alpha}} (-x_{\alpha}x_{t(\alpha)})^{-\varsigma_{\alpha}\sigma_{\alpha}}\right) \label{eq:Q_a}\\
&=& \left(\prod_{\beta_d\in {\cal E}_{\rm d}(a)}^{\sigma_\beta=1} \frac{1+x_{\beta_d}x_{\bar{\beta}_d}}{x_{\beta_d}x_{\bar{\beta}_d}}\right) \sum_{\varsigma_a \in \Sigma_{a}^{(0)}} f_a(\varsigma_a) \prod_{\alpha\in {\cal E}_{\rm d}(a)} \left( x_{\alpha}^{\varsigma_{\alpha}}\left(\varsigma_{\alpha}- \frac{x_{\alpha}x_{t(\alpha)}}{1+x_{\alpha}x_{t(\alpha)}}\right)^{\sigma_{\alpha}}\right).\label{eq:Q_a2}
\end{eqnarray}
Here, we transitioned in our notation from the general gauges $G$ to the $x$-representation; $x_a \doteq (x_{\alpha} \,| \, \alpha \in {\cal E}_{\rm d}(a))$. 
\item $z(\sigma|x)$ is a polynomial in $x$, up to the factor on the left-hand side of the first raw in Eq.~(\ref{eq:Z_sigma-x}) and when  stated in terms of $x$, thus explaining the name chosen for the representation. 
\item $Z=\sum_\sigma z(\sigma|x)$ is a constant; that is, $x$-independent, polynomial in $x$.
\end{itemize}

We observe in the next sections that a  single out $\sigma$-term, say $\sigma=0$ (chosen without loss of generality) and represented as 
\begin{eqnarray}
&& z(x)\doteq \frac{\prod\limits_{a\in{\cal V}} h_{a}(x_{a})}{\prod\limits_{\alpha \in {\cal E}} \left(1 + x_{\alpha_d}x_{\bar{\alpha}_d}\right)}, \label{eq:ZG0}
\end{eqnarray}
where $h_{a}$ is the vertex polynomial, defined in Eq.~(\ref{eq:h_a}), and  $z(x)$ is a short-cut notation for $z(0|x)$ that plays a special role in establishing known and new relations. We call $z(x)$, described by Eq.~(\ref{eq:ZG0}), the Gauge Function of the Multi-GM (\ref{eq:GM}).

\subsection{Belief Propagation Gauges}
\label{subsec:BP-G}

\begin{definition}[Interior Belief Propagation gauge \cite{06CCa,06CCb} for Soft Multi-Graph Models]
\label{def:BP-gauge}
We call a solution $x^{({\rm bp})} \in \mathbb{R}_{+}^{{\cal E}_{\rm d}}$ of the following stationary-point condition for the Gauge Function (\ref{eq:ZG0}) equations of the Soft Multi-Graph Models
\begin{eqnarray}
\forall a \in {\cal V},\quad \forall \alpha \in {\cal E}_{\rm d}(a):  \left.\partial_{x_{\alpha}} z(x)\right|_{x=x^{({\rm bp})}}=0,\label{eq:BP_0}
\end{eqnarray}
a BP gauge. 
\end{definition}

A relation between the Variational BP and BP gauge approaches is established by the following straightforward corollary of Theorem \ref{theorem:Soft-VBP}.
\begin{corollary}[Soft Multi-GM BP gauge optimality]
\label{theorem:BPG-to-VBP} 
In the case of Soft Multi-GM there exists a BP-gauge, $x$, solving Eq.~(\ref{eq:BP_0}) 
\begin{eqnarray}
\label{eq:BP-gauge-VBP} Z^{({\rm vbp})} = z(x),
\end{eqnarray}
where the Variational BP estimate for the partition function was defined in Eq.~(\ref{eq:vbp}).
\end{corollary}

Given that $z(x)$ is differentiable in $x$, BP-equations (\ref{eq:BP_0}) are well defined. Explicit version of Eq.~(\ref{eq:BP_0}), derived from Eq.~(\ref{eq:ZG0}), is: 
\begin{eqnarray}
\forall a,\quad \forall \alpha_d\in e_d(a):\quad \sum_{\varsigma_a} f_a(\varsigma_a)\left(\prod_{\beta_d\in e_d(a)} (x_{\beta_d}^{(bp)})^{\varsigma_{\beta_d}}\right)\left(\frac{x_{\alpha_d}^{(bp)}x_{\bar{\alpha}_d}^{(bp)}}{1+x_{\alpha_d}^{(bp)}x_{\bar{\alpha}_d}^{(bp)}}-\varsigma_{\alpha}\right)=0. \label{eq:BP}
\end{eqnarray}

\begin{theorem}[BP-Solution as the ``No Loose Coloring'' Condition \cite{06CCa,06CCb}]
\label{theorem:no-coloring}
The BP Eq.~(\ref{eq:BP_0}), or equivalently Eq.~(\ref{eq:BP}) can be restated in terms of the following conditions:
\begin{eqnarray}
\forall a,\quad \sum_{\alpha_d\in e_d(a)} \sigma_{\alpha_d}=1:\quad Q_a(x_a;\sigma_a)=0.
\label{eq:BP_color}
\end{eqnarray}
\end{theorem}
This result follows from the definition of $Q$ in Eq.~(\ref{eq:Q_a}) or Eq.~(\ref{eq:Q_a2}). With regard to the name chosen for the Theorem \ref{theorem:no-coloring}, it emphasizes that rewriting BP equations in the form of Eq.~(\ref{eq:BP_color}) highlights interpretation of BP in terms of the ``edge coloring''. Indeed, Eq.~(\ref{eq:BP_color}) enforces cancellation (exact zero) for all $z(\sigma|x)=z(\sigma|G)$ terms in the series on the right-hand side of Eq.~(\ref{eq:Z_inv}), where at least one node, $a$, has one of its neighboring edges, say $\alpha_d\in e_d(a)$, ``colored''; that is, it is set to $\sigma_{\alpha_d}=1$, whereas all other neighboring edges of the node, $\beta_d\in e_d(a),\ \beta_d\neq\alpha_d$, remain ``uncolored'', that is, set to, $\sigma_{\beta_d}=0$.

Solutions of the BP Eq.~(\ref{eq:BP}) (or Eq.~(\ref{eq:BP_color})) also allow a transparent interpretation in terms of the edge-consistent (but graph-globally not consistent) probability distributions. Expressions for the node- and edge- marginal probabilities evaluated at a BP gauge are 
\begin{eqnarray}
&& P^{(bp)}_a(\varsigma_a)\doteq\frac{f_a(\varsigma_a)\prod\limits_{\alpha_d\in e_d(a)}G^{(bp)}_{\alpha_d}(0,\varsigma_{\alpha_d})}{\sum\limits_{\varsigma_a}f_a(\varsigma_a)\prod\limits_{\alpha_d\in e_d(a)}G^{(bp)}_{\alpha_d}(0,\varsigma_{\alpha_d})}=\frac{f_a(\varsigma_a) \prod\limits_{\alpha_d\in e_d(a)}(x_{\alpha_d}^{(bp)})^{\varsigma_{\alpha_d}}}{\sum\limits_{\varsigma_a}f_a(\varsigma_a) \prod\limits_{\alpha_d\in e_d(a)}(x_{\alpha_d}^{(bp)})^{\varsigma_{\alpha_d}}}, \label{eq:P_bp_a}\\
&& P^{(bp)}_{\alpha}(\varsigma_{\alpha_d})\doteq G^{(bp)}_{\alpha_d}(0,\varsigma_{\alpha_d}) G^{(bp)}_{\bar{\alpha}_d}(0,\varsigma_{\alpha_d})=\frac{\left(x_{\alpha_d}^{(bp)} x_{\bar{\alpha}_d}^{(bp)}\right)^{\varsigma_{\alpha_d}}}{1+x_{\alpha_d}^{(bp)} x_{\bar{\alpha}_d}^{(bp)}} .\label{eq:P_bp_alpha}
\end{eqnarray}
Edge-consistency of the marginal beliefs means 
\begin{eqnarray}
&& \forall a,\quad \forall \alpha_d\in e_d(a),\quad \forall \varsigma_{\alpha_d}=\{0,1\}:\nonumber\\
&& P^{(bp)}_{\alpha}(\varsigma_{\alpha_d})=\frac{\left(x_{\alpha_d}^{(bp)} x_{\bar{\alpha}_d}^{(bp)}\right)^{\varsigma_{\alpha_d}}}{1+x_{\alpha_d}^{(bp)} x_{\bar{\alpha}_d}^{(bp)}}=\frac{\sum\limits_{\varsigma_a\setminus \varsigma_{\alpha_d}} f_a(\varsigma_a) \prod\limits_{\beta_d\in e_d(a)}(x_{\beta_d}^{(bp)})^{\varsigma_{\beta_d}}}{\sum\limits_{\varsigma_a}f_a(\varsigma_a) \prod\limits_{\gamma_d\in e_d(a)}(x_{\gamma_d}^{(bp)})^{\varsigma_{\gamma_d}}}=\sum_{\varsigma_a\setminus \varsigma_{\alpha_d}} P^{(bp)}_a(\varsigma_a).
\label{eq:P_bp_consistency}
\end{eqnarray}
Multiplying Eq.~(\ref{eq:P_bp_consistency}) on $\varsigma_{\alpha_d}$ and summing it up over $\varsigma_{\alpha_d}=0,1$, one arrives at the already introduced system of BP  Eqs.~(\ref{eq:BP}). We also note for consistency with earlier notations that
\begin{eqnarray}
\forall \alpha\in{\cal E}: \beta^{(bp)}_\alpha= P^{(bp)}_{\alpha}(1),\quad 
\beta^{(bp)}\doteq \left.\left(\beta^{(bp)}_{\alpha}\right|\alpha\in{\cal E}\right). \label{eq:beta-BP}
\end{eqnarray}
Notice that, consistently with the Corollary \ref{theorem:BPG-to-VBP}, Soft Multi-GM may have multiple (more then one) BP-gauges solving Eqs.~(\ref{eq:BP_0}). 

\section{Elimination of Edges} 
\label{sec:elim}

\begin{figure}
	\centering 
	\includegraphics[width=1.73in,page=4]{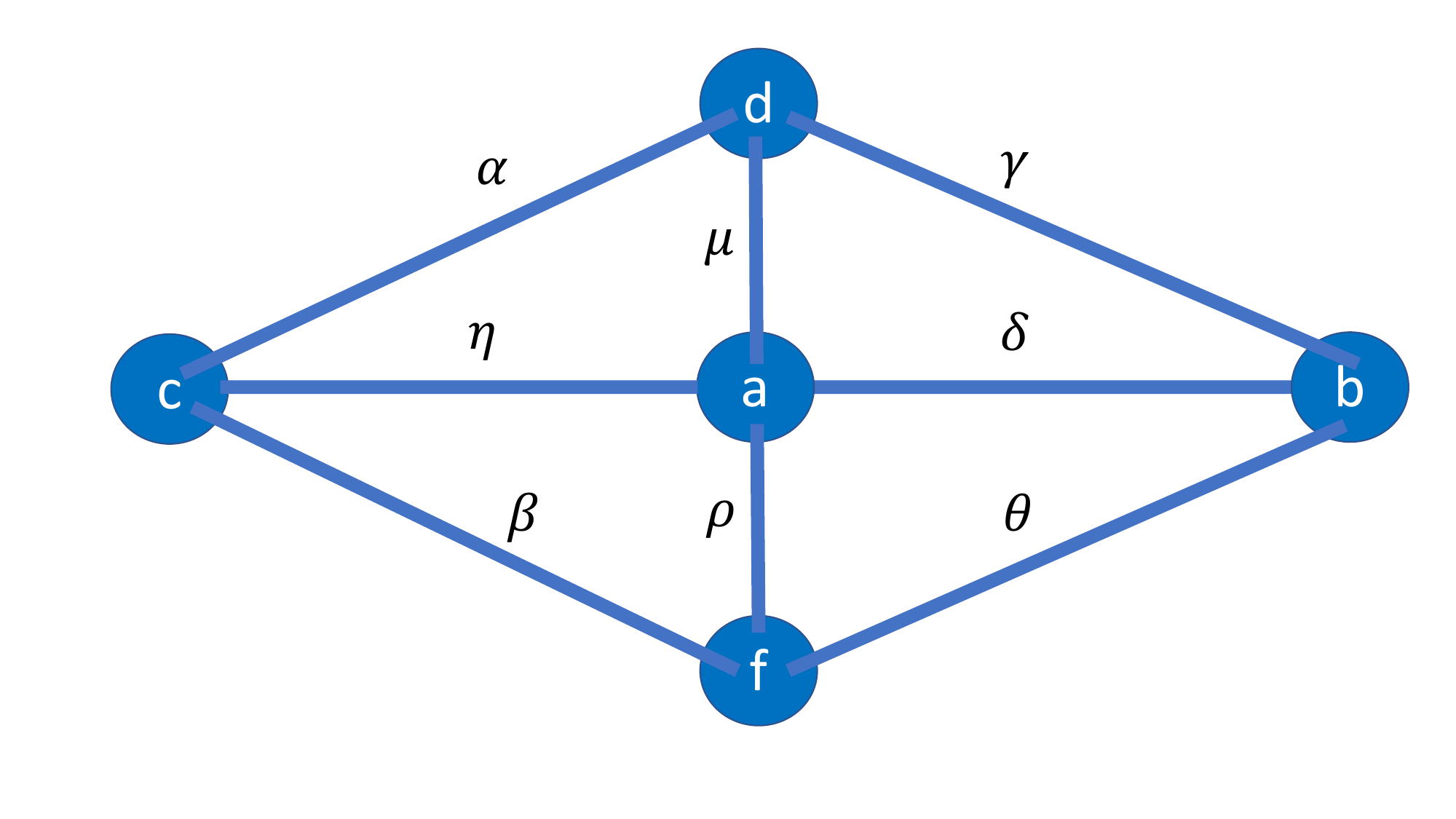}
	\includegraphics[width=1.73in,page=5]{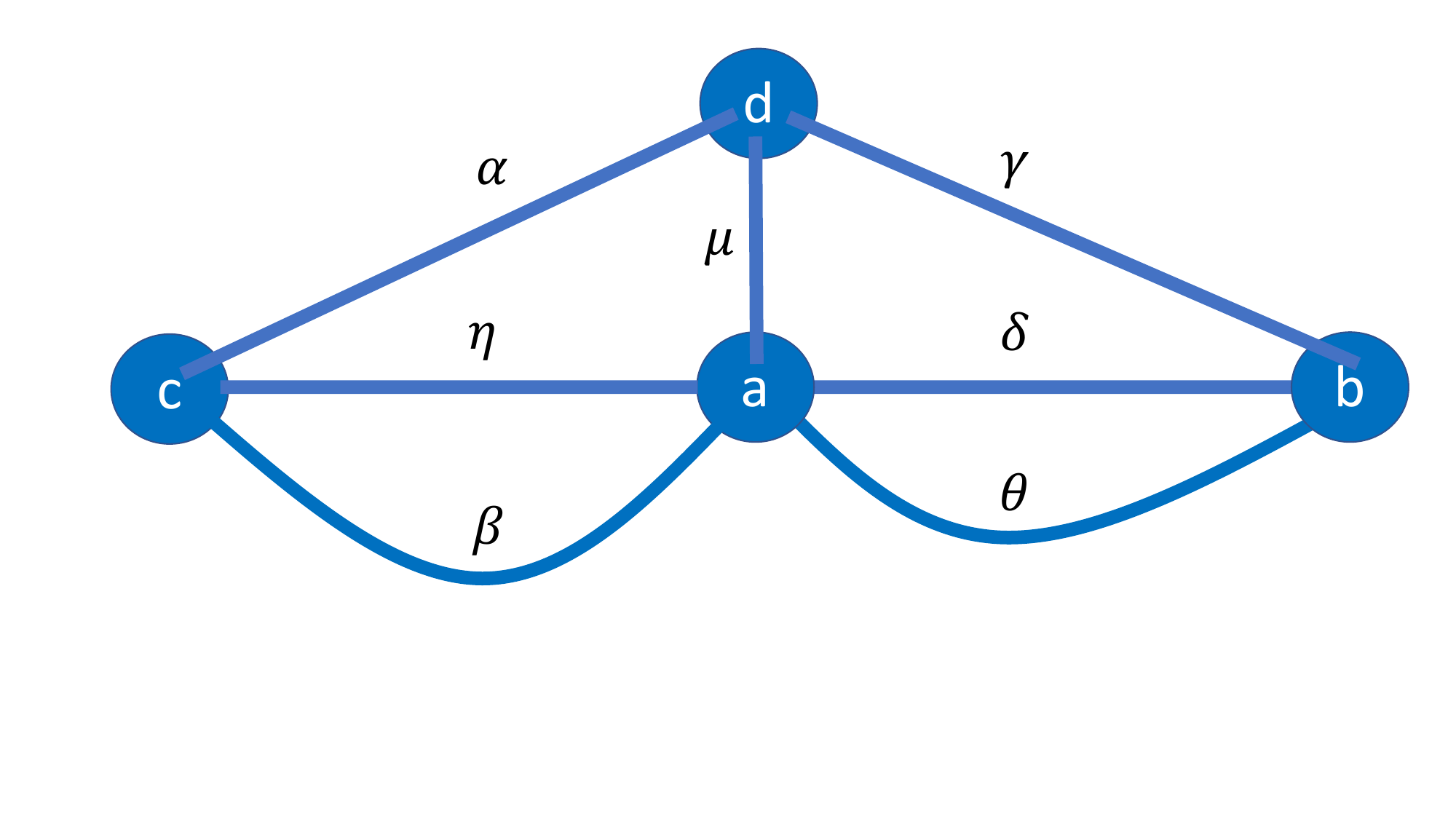}
	\includegraphics[width=1.73in,page=6]{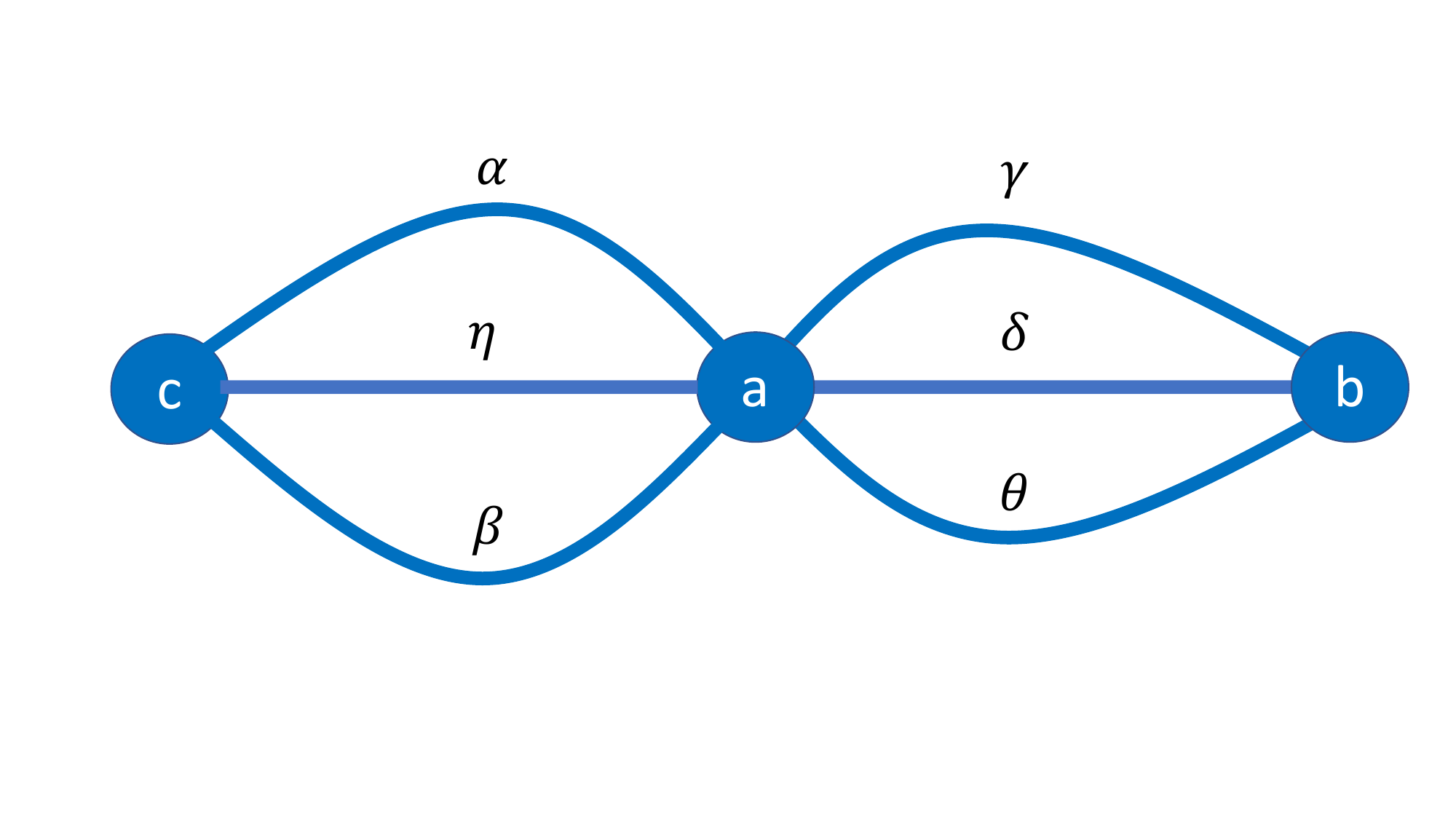}
	\includegraphics[width=1.73in,page=1]{figs/MF-GM.pdf}
	\includegraphics[width=1.38in,page=7]{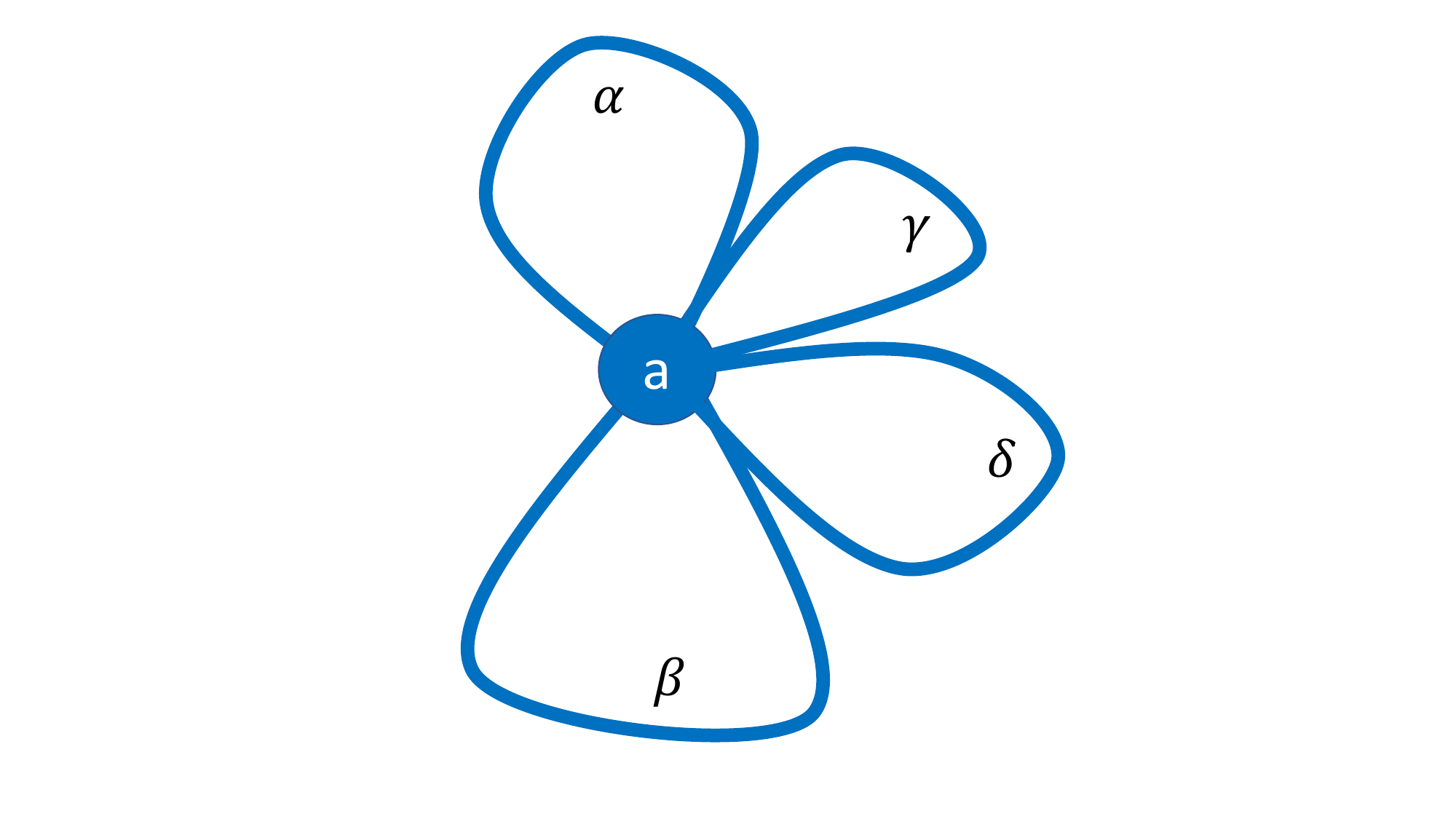}
	\includegraphics[width=1.38in,page=8]{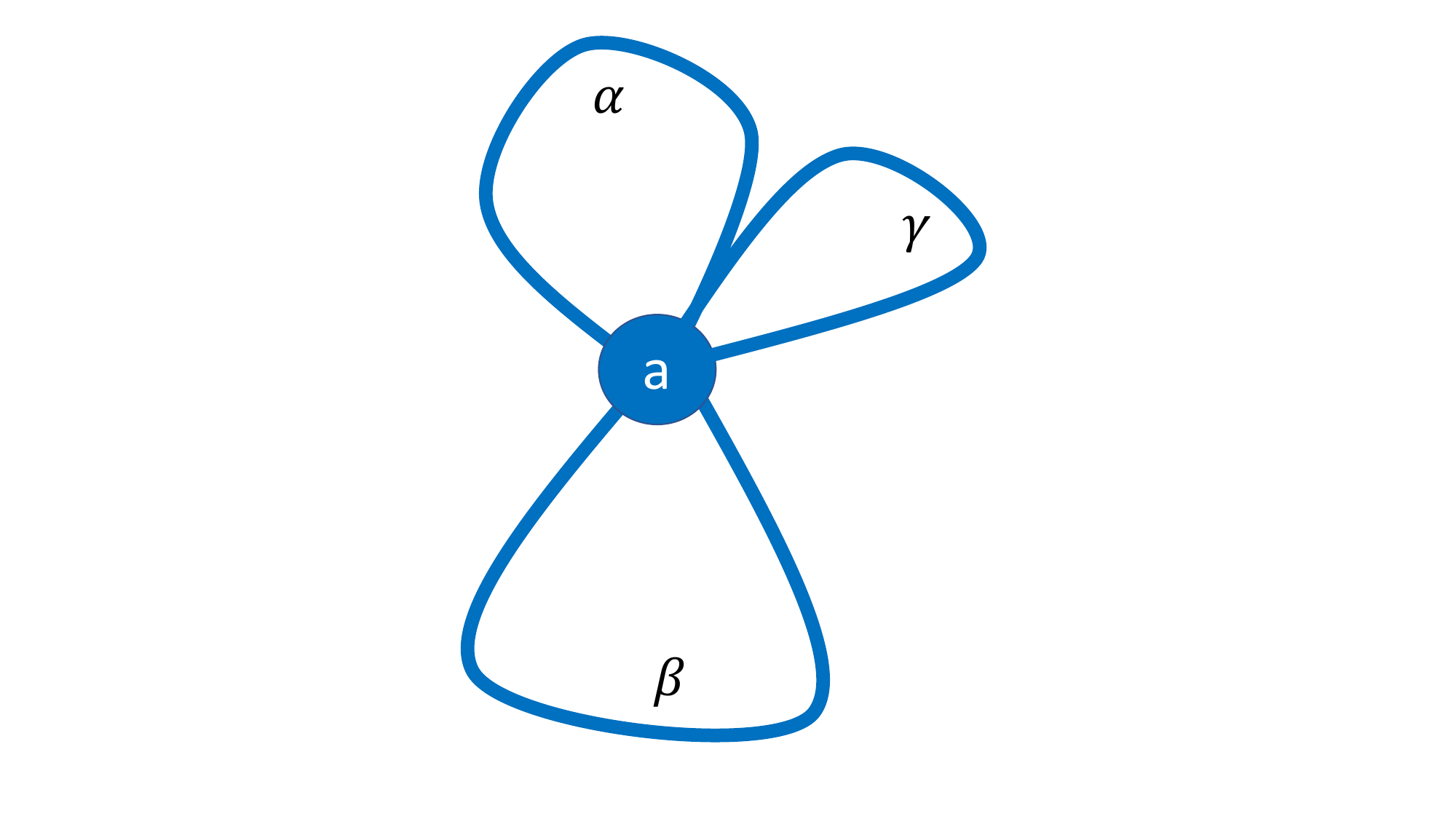}
	\includegraphics[width=1.38in,page=9]{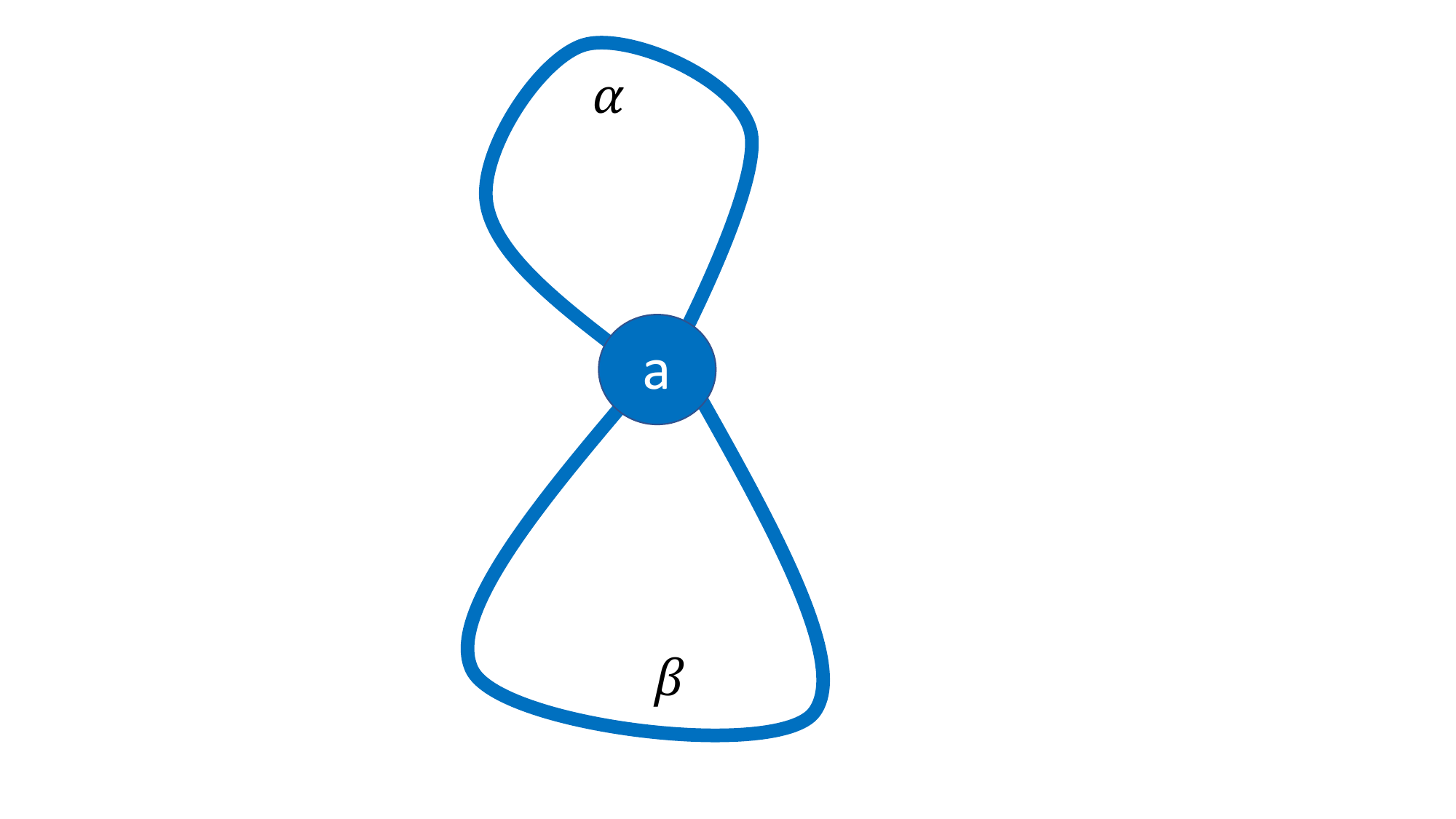}
	\includegraphics[width=1.38in,page=10]{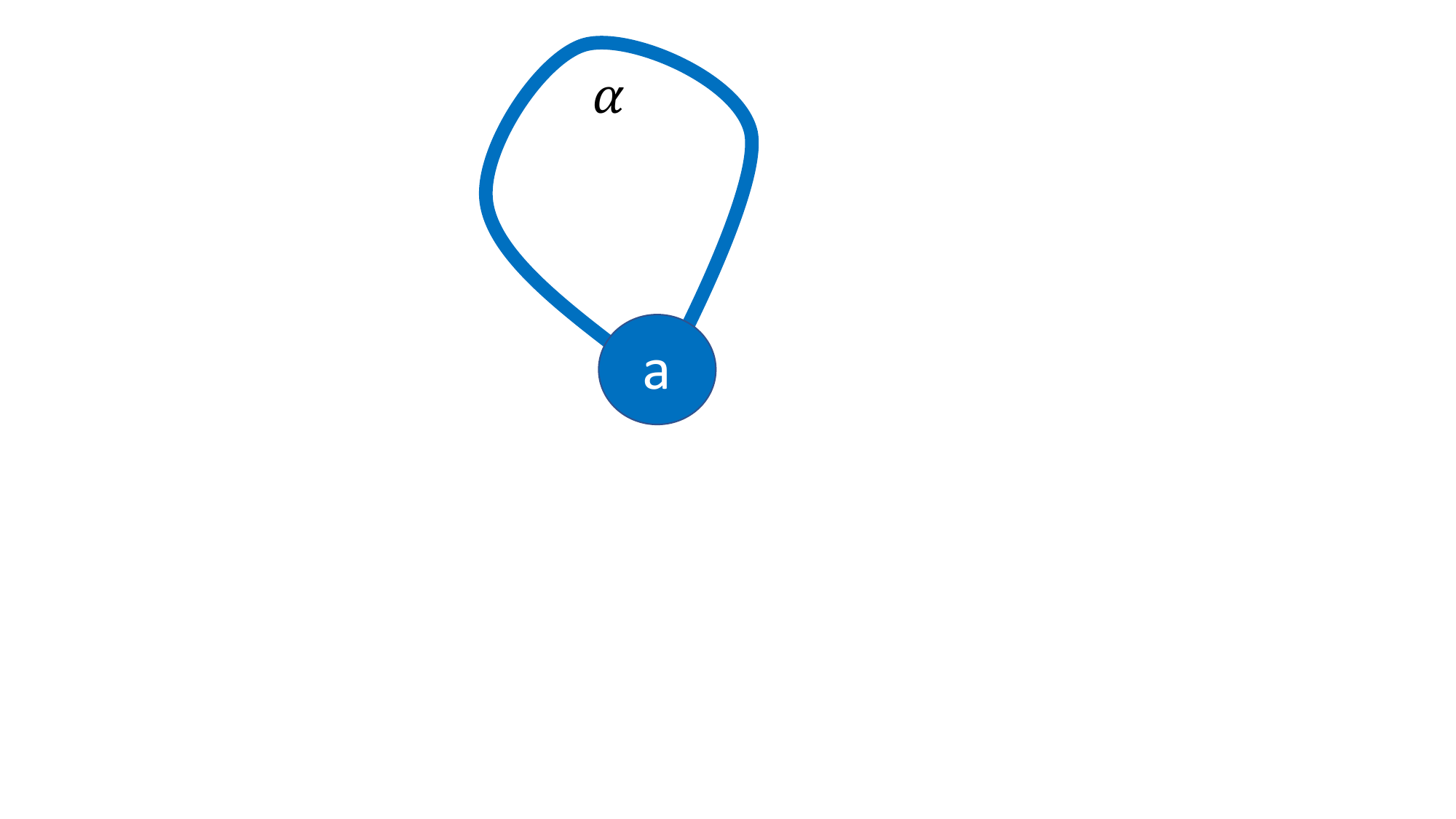}
	\includegraphics[width=1.38in,page=11]{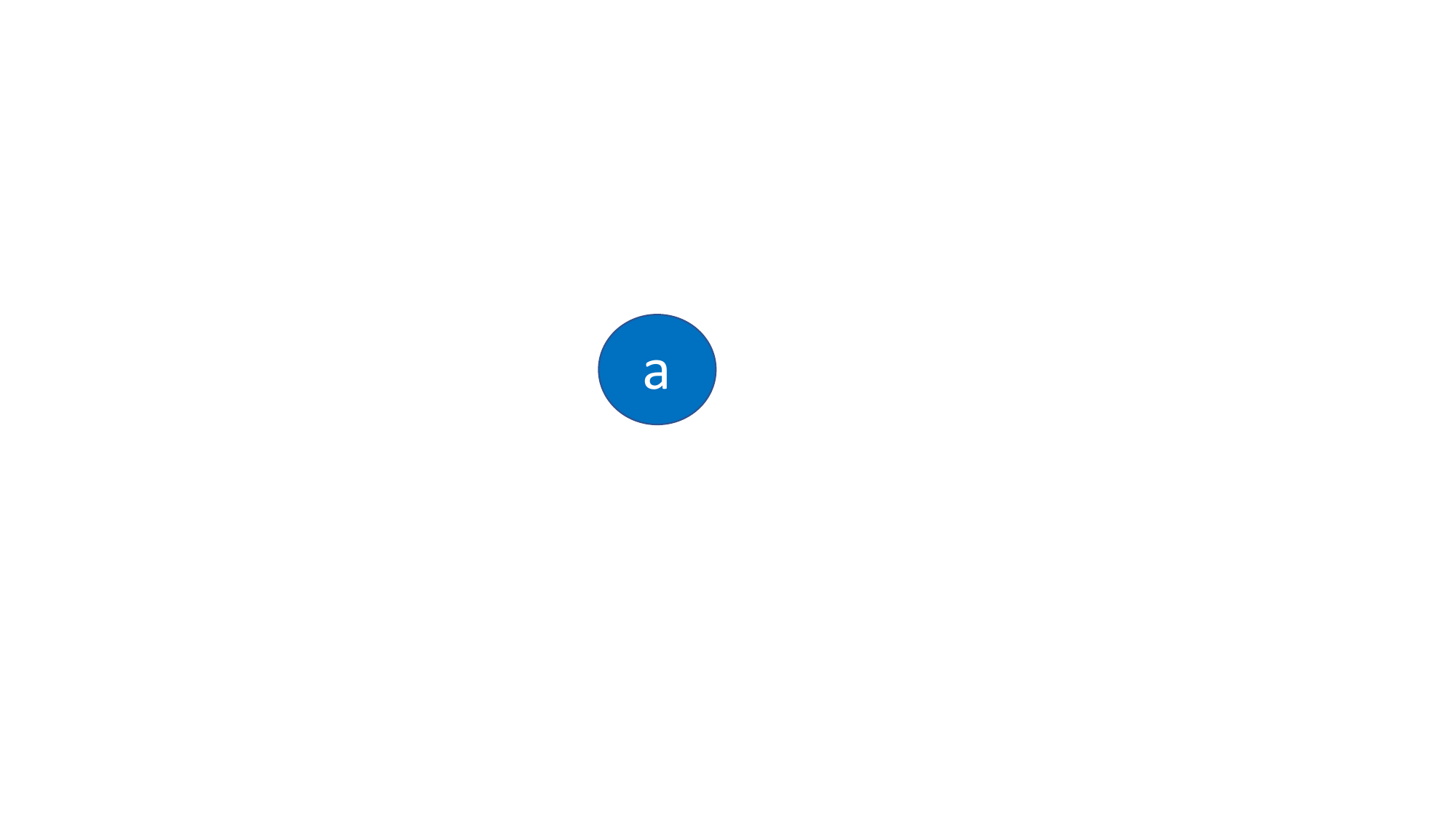}
	\caption{Graph transformation via elimination of edges is demonstrated through a sequence of sub-figures (from left to right, top to bottom). Each step results in elimination of an edge.  Eliminations of edges of two types are possible: (1) elimination of an edge connecting two nodes, e.g., as seen eliminating edges $\rho,\mu,\eta$ and $\delta$ in the first four steps in the sequence, and (2) elimination of self-edges, as seen eliminating edges $\delta,\gamma,\beta,\alpha$. \label{fig:elimination}}
\end{figure}

Assume some ordering of the graph edges,
\begin{eqnarray}
m=1,\cdots,|{\cal E}|: \alpha^{(1)},\cdots,\alpha^{(|{\cal E}|)}
\label{eq:edge-ord}
\end{eqnarray}
and consider summing up over (i.e., eliminating) edges in the expression for the Partition Function (\ref{eq:GM}) one-by-one according to this order, therefore naturally arriving at the sequence of graphs, ${\cal G}^{(0)}=({\cal V}^{(0)},{\cal E}^{(0)}),\cdots, {\cal G}^{(|{\cal E}|)}=({\cal V}^{(|{\cal E}|)},{\cal E}^{(|{\cal E}|)})$ such that 
$\forall m=1,\cdots,|{\cal E}|:\quad
{\cal E}^{(m)}\doteq {\cal E}^{(m-1)}\setminus \alpha^{(m)}$. Each next graph in the sequence has one less number of edges and the same or one less number of nodes than its predecessor, see Fig.~(\ref{fig:elimination}) for illustration. (If the number of nodes at elementary step of the sequence decreases by one,  then one of the two merged nodes, chosen arbitrarily is removed from the resulting/new set of nodes.) Then we define a sequence of Multi-GMs, on the sequence of graphs just defined, as follows
\begin{eqnarray}
& m=0: & \sigma^{(0)}=\sigma,\quad p^{(0)}(\sigma^{(0)})\doteq p(\sigma),
\nonumber \\
& \forall m=1,\cdots,|{\cal E}|:\quad &
\sigma^{(m)}=\sigma^{(m-1)}\setminus \sigma_{\alpha^{(m)}},\quad
p^{(m)}(\sigma^{(m)})\doteq \frac{f^{(m)}(\sigma^{(m)})}{Z},\quad f^{(m)}(\sigma^{(m)}) \doteq \prod_{a\in{\cal V}^{(m)}} f_a^{(m)}(\sigma_a^{(m)}),\nonumber\\ &&
\forall a\in {\cal V}^{(m)},\ a\notin v(\alpha^{(m)}):\quad f_a^{(m)}(\sigma_a^{(m)})\doteq f_a^{(m-1)}(\sigma_a^{(m-1)}),\nonumber\\
&& a\in {\cal V}^{(m)}, v(\alpha^{(m)}),{\cal V}^{(m-1)}:\quad 
f_a^{(m)}(\sigma_a^{(m)})\doteq \sum_{\sigma_{\alpha^{(m)}}} \prod_{b\in v(\alpha^{(m)}}f_b^{(m-1)}(\sigma_b^{(m-1)})
\label{eq:GM-m}
\end{eqnarray}
where $v(\alpha^{(m)})$ is defined as a set of nodes in ${\cal V}^{(m-1)}$ associated with the edge $\alpha^{(m)}$ (there may be one or two of these depending on if the edge is a self-edge or edge linking two distinct nodes). Notice also that the elimination procedure just introduced is such that the node picked in the last condition (last line) in Eq.~(\ref{eq:GM-m}) is defined uniquely.

In the following, we state a number of remarks about the elimination sequence.
\begin{itemize}
\item Notice that by construction, the Partition Function stays the same for all the Multi-GMs in the sequence. The following relations elucidate this point, $m=1,\cdots,|{\cal E}|$
\begin{eqnarray}
Z =\sum\limits_{\sigma} f(\sigma)=
\sum \limits_{\sigma_{\alpha^{(|{\cal E}|)}}}\cdots \sum \limits_{\sigma_{\alpha^{(1)}}} 
\prod_{a\in{\cal V}} f_a(\sigma_a)=\sum \limits_{\sigma_{\alpha^{(|{\cal E}|)}}}\cdots \sum \limits_{\sigma_{\alpha^{(m)}}} 
\prod_{a\in{\cal V}} f_a(\sigma_a)=\sum\limits_{\sigma^{(m)}} f^{(m)}(\sigma^{(m)}).
\label{eq:Z-inv-m}
\end{eqnarray}

\item As illustrated in Fig.~(\ref{fig:elimination}), the elimination procedure may lead to double-, triple-, and in general multiple-edges connecting the same nodes, and it may also result in self-edges, even if the original graph is a normal/simple graph. In fact,  this observation explains why we choose to work in this manuscript with the general Multi-GMs.

\item Even though we are free to choose any edge-elimination sequence, it is reasonable to eliminate, first, all normal edges (having two distinct nodes associated with the edge), as in the illustrative example of Fig.~(\ref{fig:elimination}). This results in a ``bouquet'' (of self-edges) graph containing a single node and multiple self-edges. Notice that if the original graph is a tree the resulting bouquet graph contains no self-edges, and then in this case the elimination sequence is completed. 
It is straightforward to check that the number of self-edges in the bouquet is invariant of the the normal portion of the edge elimination procedure, i.e. it does not depend on the order of the normal edge eliminations. Moreover this number is exactly equal to the number of edge cuts one needs to apply to the original graph to turn it into a tree. In the following (and unless specify otherwise) we will be assuming that the normal edges are eliminated in the elimination sequence first. 
Significance of the bouquet graph for our procedure will become clear in the following when we analyze application of the BP procedure to Multi-GMs from the sequence.

\item Obviously the exact elimination (\ref{eq:GM-m}) is not practical because the factor degree and most importantly the complexity of computing the factors grow exponentially with $m$. 

\item A number of approximate elimination schemes was introduced in the past to bypass the hardness of the exact computations of $Z$. Of these approximations the mini-bucket elimination schemes \cite{dechter2003mini,liu2011bounding,18ACSW,18ACWS} are arguably the most popular and also related to the gauge Gauge Transformation and BP subjects.  See Section \ref{sec:conclusion} for additional discussions. 

\item Even though the sequence of Multi-GMs is not tractable, it is still of a theoretical interest (e.g., in relation to analyzing the class of Multi-GMs where one can derive tractable bounds on Gauge Function, to analyze an intermediate Multi-GM in the sequence from the perspective of the Gauge Transformation and BP). This approach is explored in the following section. 
\end{itemize}

\section{From Gauge Function, $z(x)$, to Partition Function, $Z$}
\label{sec:mix_der}

\begin{theorem}[``Differentiate+marginalize"] 
\label{theorem:Diff+Marg}
Exact Partition Function, $Z$, of the GM can be recovered from the Gauge Function,  $z(x)$, defined in Eq.~(\ref{eq:ZG0}), via application of the following mixed-derivative operator:
\begin{eqnarray}
    & m=0:\quad& x^{(0)}\doteq x,\quad \Zeta^{(0)}(x)\doteq z(x),\label{eq:seq_0}\\
    & m=1,\cdots,|{\cal E}|:\quad &  x^{(m)}\doteq x^{(m-1)}\setminus \{x_{\alpha_d^{(m)}},x_{\bar{\alpha}_d^{(m)}}\},\nonumber\\ && \Zeta^{(m)}(x^{(m)})\doteq  \left.\left(1+\partial_{x_{\alpha_d^{(m)}}}\partial_{x_{\bar{\alpha}_d^{(m)}}}\right)\left(\left(1+x_{\alpha_d^{(m)}} x_{\bar{\alpha}_d^{(m)}}\right) \Zeta^{(m-1)}\right)\right|_{x_{\alpha_d^{(m)}}=x_{\bar{\alpha}_d^{(m)}}=0}, \label{eq:seq_m}\\
    &m=|{\cal E}|:\quad & x^{(|{\cal E}|)}=\emptyset, \quad \Zeta^{(|{\cal E}|)}=Z,
        \label{eq:seq_last} 
\end{eqnarray}  
where ordering of edges according to Eq.~(\ref{eq:edge-ord}) is assumed.
\end{theorem}
The statement of Theorem \ref{theorem:Diff+Marg} expressed in Eqs.~(\ref{eq:seq_0},\ref{eq:seq_m},\ref{eq:seq_last}) is a direct consequence of the following  observation:
\begin{eqnarray}
\left.\left(1+\partial_{x_{\alpha_d}}\partial_{x_{\bar{\alpha}_d}}\right) 
\left(x_{\alpha_d}^{\varsigma_{\alpha_d}}x_{\bar{\alpha}_d}^{\varsigma_{\bar{\alpha}_d}}\right)\right|_{x_{\alpha_d}=x_{\bar{\alpha}_d}=0}= \delta\left(\varsigma_{\alpha_d},\varsigma_{\bar{\alpha}_d}\right),
\label{eq:der-to-delta}
\end{eqnarray}
where $\delta\left(\varsigma_{\alpha_d},\varsigma_{\bar{\alpha}_d}\right)$ stands for the Kronecker symbol, which returns unity when $\varsigma_{\alpha_d}=\varsigma_{\bar{\alpha}_d}$ and is zero otherwise.

Furthermore, comparing sequence of transformations in the Theorem \ref{theorem:Diff+Marg} with the edge-elimination sequence described in Section \ref{sec:elim} one arrives at the following statement.
\begin{theorem}[Equivalence of the Algebraic (differentiate+marginalize) and Graphical (edge elimination) transformations]
\label{theorem:elimination=differentiation}
$\Zeta^{(m)}(x^{(m)})$, defined in Eq.~(\ref{eq:seq_m}), is the Gauge Function introduced in Eq.~(\ref{eq:ZG0}) however applied to the 
$m$-th GM in the edge-elimination sequence defined in Eq.~(\ref{eq:GM-m}):
\begin{eqnarray}
&& \Zeta^{(m)}(x^{(m)})= \frac{h^{(m)}(x^{(m)})}{\prod_{\alpha\in{\cal E}^{(m)}} \left(1+x_{\alpha_d}x_{\bar{\alpha}_d}\right)}, \label{eq:ZG0-m}\\
&& h^{(m)}(x^{(m)})\doteq \prod_{a\in{\cal V}^{(m)}} \left(\sum_{\varsigma_a \in \Sigma_{a}^{(m)}} f_a^{(m)}(\varsigma_a)\prod_{\alpha\in {\cal E}_{\rm d}^{(m)}(a)} 
x_{\alpha}^{\varsigma_{\alpha}}\right). \label{eq:hm}
\end{eqnarray}
\end{theorem}

The sequence of the graph-algebraic transformations from the Gauge Function, $z(x)$, to the Partition Function, $Z$, introduced and discussed in this and preceding sections is the  main technical point of this manuscript. 

In the next section we relate this sequence, and each step in the sequence of edge eliminations resulting in the mapping, to BP estimations of Multi-GMs in the sequence. 

\section{BP-elimination}
\label{sec:BP-elimination}

Let us take advantage of  the rational, in $x$, structure of the Gauge Function, $z(x)$, and represent $h(x)$ in Eq.~(\ref{eq:ZG0}), as a generic quadratic function of $x_{\alpha_d}$ and $x_{\bar{\alpha}_d}$, where the edge $\alpha$ of the original graph, ${\cal G}$ is selected arbitrarily,
\begin{eqnarray}
h(x)\doteq  h^{(0,0)}+h^{(1,0)}x_{\alpha_d}+ h^{(0,1)}x_{\bar{\alpha}_d}
+h^{(1,1)}x_{\alpha_d}x_{\bar{\alpha}_d},
\label{eq:h-pol}
\end{eqnarray}
the coefficients of the expansion are non-negative, $\forall i,j=0,1,\quad h^{(i,j)}\geq 0$, and also dependent on $x^{(1)}=x\setminus \{x_{\alpha_d},x_{\bar{\alpha}_d}\}$ (the dependence is dropped here and also in some formulas below to avoid bulky expressions). Substituting Eq.~(\ref{eq:h-pol}) into the BP-equations (\ref{eq:BP_0}) for the edge $\alpha$, one arrives at a quadratic equations for $x_{\alpha_d}$ and $x_{\bar{\alpha}_d}$, which results in two roots of which only one is physical, i.e. consistent with respective positive marginal probabilities:
\begin{eqnarray}
x_{\alpha_d}^{(\alpha-{\rm bp})} &=&\frac{h^{(1,1)}-h^{(0,0)}+\sqrt{\left(h^{(1,1)}-h^{(0,0)}\right)^2+4 h^{(0,1)}h^{(1,0)}}}{2 h^{(1,0)}},\label{eq:bp-alpha}\\
x_{\bar{\alpha}_d}^{(\alpha-{\rm bp})} &=&\frac{h^{(1,1)}-h^{(0,0)}+\sqrt{\left(h^{(1,1)}-h^{(0,0)}\right)^2+4 h^{(0,1)}h^{(1,0)}}}{2 h^{(0,1)}},\label{eq:bp-bar-alpha}
 \end{eqnarray}
The value of $h(x)/(1+x_{\alpha_d} x_{\bar{\alpha}_d})$ evaluated at the physical $\alpha$-BP-gauge is    
\begin{eqnarray}
 h^{(\alpha-{\rm bp})}=
\left. \frac{h(x)}{1+x_{\alpha_d} x_{\bar{\alpha}_d}}\right|_{x_{\alpha_d}=x_{\alpha_d}^{(\alpha-bp)}, x_{\bar{\alpha}_d}=x_{\bar{\alpha}_d}^{(\alpha-bp)}} = \frac{h^{(1,1)}+h^{(0,0)}+\sqrt{\left(h^{(1,1)}-h^{(0,0)}\right)^2+4 h^{(0,1)}h^{(1,0)}}}{2}.
\label{eq:z-alpha}
\end{eqnarray}

Notice that the expression (\ref{eq:h-pol}) for, $h$,  simplifies when $\alpha$ is a normal edge (not a self-edge). In this special case $h$ is a product of two polynomials of the first order over  $x_{\alpha_d}$ and $x_{\bar{\alpha}_d}$, respectively,  i.e. in this case $h^{(1,1)} h^{(0,0)}=h^{(0,1)}h^{(1,0)}$. Then  Eqs.~(\ref{eq:bp-alpha},\ref{eq:bp-bar-alpha}) and Eq.~(\ref{eq:z-alpha}) transform to 
\begin{eqnarray}
\mbox{If }\alpha\mbox{ is a normal edge:}\quad x_{\alpha_d}^{(\alpha-{\rm bp})} =\frac{h^{(1,1)}}{h^{(1,0)}},\quad x_{\bar{\alpha}_d}^{(\alpha-{\rm bp})}
 = \frac{h^{(1,1)}}{h^{(0,1)}},\quad h^{(\alpha-{\rm bp})}= h^{(1,1)}+h^{(0,0)}.
\label{eq:edge}
\end{eqnarray}
Comparing Eq.~(\ref{eq:edge}) with the first step of the exact edge elimination (first step of the mixed derivative application to the Partition Function) described in Section \ref{sec:mix_der}, observing that the result is equivalent to the exact elimination, as long as our Multi-GM has a maximal BP gauge, which happens to be interior. Therefore, sequential BP-elimination of edges will be still equivalent to exact elimination, as long as the edge to be eliminated is normal, and the Multi-GM has a maximal BP gauge, which is interior, on each step of the elimination procedure. A formal statement is as follows.

\begin{theorem}[BP-to-bouquet]
\label{theorem:BP-to-bouquet}
Elimination of normal edges in Multi-GM via application of the sequential edge-by-edge BP-gauge procedure, resulting in Multi-GM for the bouquet graph, is exact, i.e. it is equivalent to the exact elimination via summation over binary variables associated with the eliminated edges.
\end{theorem}

Notice that the sequential BP-elimination approach does not generalize to self-edges, because in this case expression (\ref{eq:z-alpha}) for $h^{(\alpha-bp)}$ returns a fractional function of $h^{(i,j)}$, thus resulting in a fractional (not polynomial) function over the remaining gauge variables, $x\setminus \{x_{\alpha_d},x_{\bar{\alpha}_d}\}$. However, the general fractional relations (\ref{eq:bp-alpha},\ref{eq:bp-bar-alpha},\ref{eq:z-alpha}) still results in a number of useful statements.

\begin{theorem}[BP-saddle]
Any BP gauge solution of Eq.~(\ref{eq:BP_0}) is a saddle-point of the Gauge Function, $z(x)$, defined in Eq.~(\ref{eq:ZG0}), over any pair of the edge-gauges,
$x_{\alpha_d},x_{\bar{\alpha}_d}$.
\end{theorem}
\begin{proof}
We only need to discuss here the case of an interior gauge, extending it to the (generic) case of a BP gauge following the logic of Section \ref{subsec:BP-G}. Expanded in the Taylor series over deviations from the BP-gauge,
the rational expression, $h(x)/(1+x_{\alpha_d}x_{\bar{\alpha}_d})$, representing the $x_{\alpha_d},x_{\bar{\alpha}_d}$-dependent part of $z(x)$, where $h(x)$ is from Eq.~(\ref{eq:h-pol}), becomes
\begin{eqnarray}
&& h(x)-h^{(\alpha-{\rm bp})}- \mbox{ cubic corrections}=d \left( x_{\alpha_d}-x_{\alpha_d}^{({\rm bp})}\right)\left( x_{\bar{\alpha}_d}-x_{\bar{\alpha}_d}^{({\rm bp})}\right)\nonumber\\ && =\frac{d}{4} \left(\left( x_{\alpha_d}+x_{\bar{\alpha}_d}-x_{\alpha_d}^{({\rm bp})}-x_{\bar{\alpha}_d}^{({\rm bp})}\right)^2- \left( x_{\alpha_d}-x_{\bar{\alpha}_d}-x_{\alpha_d}^{({\rm bp})}+x_{\bar{\alpha}_d}^{({\rm bp})}\right)^2\right),\label{eq:z-exp}
\end{eqnarray}
where $d>0$ (we skip presenting here bulky but explicit expression for $d$), thus completing the proof. 
\end{proof}

The following technical statement, proven through a straightforward algebraic manipulation, introduces another useful feature of Eq.~(\ref{eq:z-alpha}).
\begin{lemma}[BP- vs exact- reductions]
\label{lemma:BP-reduction}
Consider a generic polynomial, $h(x)$, representing Multi-GM and stated in the form of expansion (\ref{eq:h-pol}) over variables $x_{\alpha_d}$ and $x_{\bar{\alpha}_d}$ associated with the edge, $\alpha$. Then,  condition that for all values of the variables remaining after contraction of the edge $\alpha_d$, the BP-reduced function, defined according to Eq.~(\ref{eq:z-alpha}), is less or equal then the exact-reduced function, 
\begin{eqnarray}
\forall x^{(1)}:\quad 
\left. \frac{h(x)}{1+x_{\alpha_d} x_{\bar{\alpha}_d}}\right|_{x_{\alpha_d}=x_{\alpha_d}^{(\alpha-{\rm bp})}, x_{\bar{\alpha}_d}=x_{\bar{\alpha}_d}^{(\alpha-bp)}}\leq h^{(1,1)}(x^{(1)})+h^{(0,0)}(x^{(1)}),
\label{eq:BP-less-exact}
\end{eqnarray}
holds if
\begin{eqnarray}
\forall x^{(1)}:\quad h^{(0,1)}(x^{(1)})h^{(1,0)}(x^{(1)})\leq h^{(0,0)}(x^{(1)})h^{(1,1)}(x^{(1)}).
\label{eq:reduced-ineq}
\end{eqnarray}
\end{lemma}

The following optimization version of the Lemma \ref{lemma:BP-reduction} was also introduced (and proven) in \cite{17AG,17SVa}. 
\begin{lemma}[BP- vs exact- reductions: variational version]
\label{lemma:BP-reduction-opt}
Given representation (\ref{eq:h-pol}) of $h(x)$ as the polynomial in $x_{\alpha_d},x_{\bar{\alpha}_d}$, an arbitrarily chosen marginal belief, $\beta_{\alpha}\in [0,1]$, and Eq.~(\ref{eq:reduced-ineq}) satisfied, guarantees that
\begin{eqnarray}
\forall x^{(1)}:\quad (\beta_{\alpha})^{\beta_{\alpha}}(1-\beta_{\alpha})^{1-\beta_{\alpha}} \inf_{x_{\alpha_d},x_{\bar{\alpha}_d}}\frac{h(x)}{\left(x_{\alpha_d}x_{\bar{\alpha}_d}\right)^{\beta_{\alpha}}}\leq h^{(1,1)}(x^{(1)})+h^{(0,0)}(x^{(1)}).
\label{eq:BP-exact-variational}
\end{eqnarray}
\end{lemma}

\section{Bi-Stability and Monotonicity of BP Elimination}
\label{sec:bistab-monotone}

Remarkably the condition (\ref{eq:reduced-ineq}) was shown to hold generically if $h(x)$ is Bi-Stable \cite{17AG},  where the stability and bi-stability of a polynomial are defined as follows. 
\begin{definition}[Real Stable Polynomial and Bi-Stable Polynomial. See \cite{17AG}]
	\label{def:BS}
	A nonzero polynomial, $g(x)\in\mathbb{R}[x_1,\cdots,x_N]$, with real coefficients is Real Stable if none of its roots $z=(z_1,\cdots,z_N)\in \mathbb{C}^N$ (i.e., solutions of $g(z)=0$) satisfies: $\mbox{Im}(z_i)>0$ for every $i=1,\cdots,N$. A polynomial
	$h(x_{\alpha_d^{(1)}},x_{\bar{\alpha}_d^{(1)}};x_{\alpha_d^{(2)}},x_{\bar{\alpha}_d^{(2)}}\cdots)$ is Bi-Stable if $h(x_{\alpha_d^{(1)}},-x_{\bar{\alpha}_d^{(1)}};x_{\alpha_d^{(2)}},-x_{\bar{\alpha}_d^{(2)}}\cdots)$.
\end{definition}

Therefore we arrive at the following powerful statement. 
\begin{theorem}[Monotonicity of Variational BP]
    \label{theorem:VBP-monotone}
Consider an Multi-GM over graph ${\cal G}$ and with the factors correspondent to a Bi-Stable polynomial, $h(x)$, build a sequence of Multi-GMs, $m=0,\cdots,|{\cal E}|$, starting with the original Multi-GM and getting next Multi-GM in the sequence by contraction of an edge, and denote (according to notations of the preceding Sections), graph, vector of gauge variables, polynomial and Variational BP estimation for Partition Function evaluated at the $m$-th step of the hierarchy, ${\cal G}^{(m)}, x^{(m)}, h^{(m)}(x^{(m)})$ and $Z^{(k; {\rm vbp})}$, respectively. Then
\begin{itemize}
    \item[(1)] Each polynomial, $h^{(m)}(x^{(m)})$, in the sequence is Bi-Stable. 
    \item[(2)] Value of the Variational BP estimation for Partition Function does not decrease with elimination and therefore 
    \begin{eqnarray}
    Z^{(\rm{vbp})}=Z^{(0; {\rm vbp})}\leq Z^{(1; {\rm vbp})} \ldots \leq Z^{(|{\cal E}|; {\rm vbp})}=Z.
    \label{eq:Zvbp<Z}
    \end{eqnarray}
\end{itemize}
\end{theorem}
\begin{proof}
The first step of exact contraction, applied to $h(x)$, consists of applying a differential operator $\left(1 + \partial_{x_{\alpha_d^{(1)}}} \partial_{x_{\bar{\alpha}_d^{(1)}}}\right)$, followed by setting both $x_{\alpha_d^{(1)}}$ and $x_{\bar{\alpha}_d^{(1)}}$ to zero. The composite operator preserves Bi-Stability. This follows from a standard argument that involves characterization of linear operators $T$ that preserve real stability of polynomials in terms of their algebraic symbols \cite{08BB,09BB}. The algebraic symbol of the above composite operator is easily computed, and the stability of its symbol is obvious, therefore $h^{(1)}(x^{(1)})$ is Bi-Stable. Applying the logic sequentially, we conclude that all polynomials in the sequence: $m=1,\cdots,|{\cal E}|,\quad h^{(m)}(x^{(m)})$ are Bi-Stable. Statement (1) of the Theorem \ref{theorem:VBP-monotone} is proven. 

Applying Lemma \ref{lemma:BP-reduction-opt} to each elimination in the sequence one writes
\begin{eqnarray}
&& m=1,\cdots,|{\cal E}|\quad \forall \beta^{(m-1)}\in [0,1],\quad \forall x^{(m)}:\nonumber\\  &&  
 (\beta_{\alpha^{(m)}})^{\beta_{\alpha^{(m)}}}(1-\beta_{\alpha^{(m)}})^{1-\beta_{\alpha^{(m)}}}
\inf\limits_{x_{\alpha_d^{(m)}},x_{\bar{\alpha}_d^{(m)}}>0}
\frac{h^{(m-1)}(x^{(m-1)})}{ \left(x_{\alpha_d^{(m)}} x_{\bar{\alpha}_d^{(m)}}\right)^{\beta_{\alpha^{(m)}}}} \leq h^{(m)}(x^{(m)}).
\label{eq:BP-exact-seq}
\end{eqnarray}
Next one multiplies both sides of Eq.~(\ref{eq:BP-exact-seq}) on,
\begin{eqnarray}
\frac{\prod_{\alpha\in{\cal E}^{(m)}} (\beta_{\alpha})^{\beta{\alpha}}(1-\beta_{\alpha})^{1-\beta{\alpha}}}{\prod_{\alpha\in{\cal E}^{(m)}} (x_{\alpha_d}x_{\bar{\alpha}_d})^{\beta_{\alpha}}},
\label{eq:multiplier}
\end{eqnarray} 
and observe that $\inf_{x^{(m)}>0}$ applied to the left hand side of the resulting inequality is less or equal to the $\inf_{x^{(m)}>0}$ applied to the right hand side of the inequality. Finally, similar application of the $\max_{\beta^{(m-1)}}$ operation to the two sides of the inequality obtained at the previous step results in the desired Eq.(\ref{eq:Zvbp<Z}). (2) is proven.
\end{proof}
Notice that related technical statements and proofs were reported in \cite{17AG} and \cite{17SVa}.

\section{Discussion and Path  Forward}
\label{sec:conclusion}

Inspired by \cite{17AG,17SVa}, we began this manuscript by generalizing the Bethe Free Energy approach from normal GM to multi-GM. Then we  
reformulate gauge representation  of \cite{06CCa,06CCb} for computing Partition Function of an Multi-GM in terms of polynomials. According to \cite{06CCa,06CCb}, picking up a Gauge Function, which is a term in the gauge-transformed series, and making it least sensitive to the gauge transformations (looking for  stationary point of the Gauge Function over gauges) results in the BP gauge  and subsequently in the Loop Series expression for the Partition Function, where each term is an explicit functional of the BP gauge. One may say that the algebraic essence of the Loop Series approach is in reconstructing exact Partition Function from its tractable BP approximation by summing the Loop Series terms. The main construct of this manuscript is an alternative map, suggested by analogy with the polynomial construct of \cite{09Gur,17AG,17SVa} from the Gauge Function to the Partition Function, $Z$. Now, this is possible via a sequence of differentiation of the Gauge Function over gauge variables, each associated with a directed edge of the graph. We show that a differentiation step in the sequence can be interpreted graphically as contraction/elimination of an edge, which results in a new Multi-GM with one less edge and one less node. Partition Function of each Multi-GM in the sequence is exactly equal to Partition Function of the original Multi-GM. (Note in passing that (a) construction is similar to an elementary transformation step in the graph minor theory \cite{05DHK}; and (b) even if the original GM is normal, i.e. it contains only normal edges and no self-edges, one eventually arrives advancing in the sequence at an Multi-GM, containing self-edges, therefore justifying discussion of the most general Multi-GM setting.) Evaluating minimum of the Bethe Free Energy, or equivalently specially defined optimum of the respective Gauge Function,  for each Multi-GM in the sequence we get an optimal BP estimation for each Multi-GM in the sequence. We observe that BP transformation is exact for contraction of a normal edge but approximate for contraction of a self-edge. Then, utilizing the power of the Real Stable Polynomials theory \cite{11Pem,11Wag,13Vis}, we showed that (a) all polynomials associated with factors of the contracted Multi-GMs  are Real Stable Polynomials if all polynomials associated with factors of the original Multi-GM are Real Stable Polynomials; (b) optimal BP estimation for Partition Function of an Multi-GM in the sequence upper bounds optimal BP estimation for Partition Function of the preceding Multi-GM (in the sequence). Corollary of the latter statement is a new proof (also generalization from GM to Multi-GM) that the optimal BP estimation of the original Multi-GM low bounds the exact Partition Function. The original proof was made for the special case of bi-partite GM in \cite{17SVa}, when Real Stable Polynomials is reduced to Real Stable, while the polynomial version of our results is a particular case of the relation presented in  \cite{17AG}. 

Synthesis of the two approaches, GM/gauges/BP/loops and Real Stable Polynomials, is far from explored by this and preceding \cite{11Gur,17AG,17SVa} manuscripts. Therefore, we  find it useful to combine in the remainder of this section some remarks, that follow from the manuscript results, with speculations about future research directions.
\begin{itemize}

\item \underline{\bf Bi-Stable examples}: Linear (degree one) real polynomials, $a+\sum_i b_i z_i$,  with $a>0$ and $\forall i, b_i>0$, correspondent to generic matching (monomer--dimer) models over bi-partite graphs, is the main example of a GM represented by Bi-Stable Polynomials/Real Stable Polynomials. 
A non-bi-partite Bi-Stable Polynomials example can be derived from the bi-partite case by contraction of (a number of) edges described in Section \ref{sec:elim}. Other known classes of BiStable Polynomialss are also to be explored in GMs.  In particular, determinantal polynomials, $\det(B+\sum_i z_i A_i)$, with positive semi-definite matrices, $\forall i:\ A_i >0$, and Hermitian matrix $B$ (all matrices are quadratic of the same dimensionality) is another (and arguably the most popular example in the Real Stable Polynomials theory) that may also have interesting relations/consequences for Fermion GM of statistical and quantum physics, see \cite{08CCa,08CCb} and references therein. All statements made in this manuscript (e.g., on the ordering of the Partition Function estimates for the contracted sequence of Multi-GMs) would apply to the special Multi-GM with the underlying Bi-Stable Polynomials structure. 

\item \underline{\bf Improving BP approximation}: The elimination scheme of Section \ref{sec:elim} has a significant approximation potential, both theoretically and empirically.  On the theoretical side, one may attempt to seek a more restrictive class of polynomials, for example, models which are Bi-Stable Polynomials locally and not globally in the upper-half planes for each complex variable (associated with a directed edge). Approached from an empirical/algorithmic stand point, the elimination can be carried over and then checked post-factum (if it results in an increase or decrease of the Partition Function of the contracted graphs). Given that the complexity of the contracted Multi-GM evaluations will be increasing exponentially with the elimination steps, one may consider approximate methods in the spirit of the mini-bucket elimination schemes \cite{dechter2003mini,liu2011bounding,18ACSW,18ACWS}. Therefore, developing new mini-bucket schemes based on the polynomial stability properties is another promising direction for the future. Besides, it will be important to take advantage of the polynomial structure in creating synthetic practical algorithms mixing BP/Gauge Transformation/Loop Calculus ideas with random sampling ideas; for example, in the spirit of Fully Polynomial Randomized Approximation Schemes and empirical schemes a-la \cite{16ASS}, and the mini-bucket elimination schemes a-la \cite{18ACSW,18ACWS}.

\item \underline{\bf Efficient computation of BP gauge}: A comment in Section 3.3 of \cite{17SVa} suggests that some algorithmic improvements for computing $Z^{(bp)}$ based on techniques from the theory of stable polynomials  are possible. In general, an Real Stable feature of the node polynomials does not guarantee convexity of the Bethe Free Energy (\ref{eq:vbp}), even though for some special cases and noticeably for the case of perfect matching \cite{10Von,14Lel,15Lel}, the convexity may be guaranteed. Moreover, an optimal solution of the Bethe Free Energy may be achieved at the boundary of the belief polytope, thus not satisfying the BP Eq.~(\ref{eq:BP}). Since results of this manuscript are dependent on the existence of a valid solution of BP Eq.~(\ref{eq:BP}), it is imperative for future progress to develop Real Stable Polynomials theory-based schemes answering the question of existence and discovering solution(s) of BP equations efficiently. It will also be important to generalize the analysis of this manuscript to the case when solution of the BP Eqs.~(\ref{eq:BP}) is found outside of the feasibility domain (outside of the BP polytope).

\item \underline{\bf Higher alphabets and higher-degree polynomials}: Both the gauge transformation and the Real Stable Polynomials theory extend, in principle, to the case of higher alphabets and related higher-degree polynomials. The loop tower approach of \cite{07CC} and alternative approach of \cite{15Mor} build generalizations of the Gauge Transformation and Loop Series for the case of higher (than binary) alphabets. We conjecture that choosing polynomial gauge parametrization for the non-binary cases, generalizing elimination/differentiation procedure such that it would result in a Multi-GM sequence with the desired non-decreasing BP estimates for the Partition Function, is possible.

\item \underline{\bf Synthesis with Fully Polynomial Deterministic Algorithmic Schemes}: A number of Fully Polynomial Deterministic Algorithmic Schemes that apply ideas from the theory of graph polynomials to Partition Function were recently developed. Some of the most recent results have focused on estimating (a) Partition Function of cliques over graphs  \cite{15Bar,16Bar_book}; (b) permanents of some complex matrices \cite{16Bar,16Bar_book}; (c) complex-valued graph polynomials on finite degree graphs, including Tutte polynomials, independence polynomials, as well as Partition Function of complex valued spin and edge coloring models \cite{17PR}; and (d) Partition Function of attractive Ising models of bounded degree \cite{17LSS}. All of the manuscripts just mentioned rely on results of advance complex analysis initiated by studies of phase transitions for infinite systems in statistical physics \cite{52LY,52YL,72HL,05SS}, which allowed zeros of graph polynomials to be located and satisfy certain properties. Polynomials considered in the studies are special but also different from Real Stable polynomials considered in \cite{17SVa} and in this manuscript. Besides, and as was emphasized in \cite{17SVa}, there exist models that are not of the class explained by Real Stable Polynomials, notably attractive Ising models \cite{07SWW,12Ruo,17Ruo}, even though their Partition Functions are bounded from below by respective BP estimates. (Note in passing that light may be shed on the relation between the two seemingly unrelated statements of BP validity as a lower bound via mapping of a general Ising model to a matching model suggested in \cite{72HL}.) It will be important to reconcile and unify these sister subjects. We conjecture that a combination of methods from the gauge and graph transformations, loop series, and analysis of Partition Function zeros in the complex domain of parameters will be imperative for making the progress towards unification of the existing approaches.

\end{itemize}

\section{Acknowledgements}  

MC and YM are grateful to organizers and participants of the EPFL, Bernoulli center, workshops on ``Introduction to Partition Functions'' in July of 2018 and "Applications of the Partition Functions" in November of 2018, where this work was initiated and where its first version was criticized, respectively. We are particularly indebted to Nisheeth Vishnoi for many discussions and useful explanations, to Peter Csikvari for attracting our attention to inconsistency in our early notes and to Nima Anari explaining to us the notion of bi-stability. We are also thankful to Pascal Vontobel, Jinwoo Shin, and Marc Lelarge for help with references and useful comments. The work at LANL was carried out under the auspices of the National Nuclear Security Administration of the U.S. Department of Energy under Contract No. DE-AC52-06NA25396. The work was partially supported by DOE/OE/GMLC and LANL/LDRD/CNLS projects.

\appendix

\section{Loop Series \cite{06CCa,06CCb} restated  in the polynomial form}
\label{app:LS}

With  gauge $x$ chosen to satisfy the BP Eqs.~(\ref{eq:BP_0}), or equivalently Eqs.~(\ref{eq:BP_color}), of Soft Multi-GM, thus denoted $x^{(bp)}$, consistently with notations introduced in the main part of the manuscript, one derives from Eq.~(\ref{eq:Z_inv}) the Loop Series, expression for $Z$:
\begin{eqnarray}
&& Z=\sum_{\sigma\in\Sigma_{glp}} z(\sigma|x^{(bp)}),\label{eq:LS}
\end{eqnarray}
where $\Sigma_{glp}$ stands for the set of $\sigma$ vectors corresponding to the so-called Generalized Loops , $\sigma\in \Sigma_{glp}\mbox{ iff } \forall a\in{\cal V},\ \sum_{\alpha\in e(a)}\sigma_\alpha\neq 1$. Note that an empty set, $\sigma=0^{|{\cal E}|}$ is included in $\Sigma_{glp}$. A Soft Multi-GM can also be thought of as a subgraph of ${\cal G}$, ${\cal G}^{(\sigma)}=({\cal V}^{(\sigma)},{\cal E}^{(\sigma)})\subseteq {\cal G}$, constructed by coloring edges of the graph (setting respective $\sigma_\alpha$ to unity) according to the following rules: each node neighboring an edge of the Soft Multi-GM set contains at least two edges colored, i.e. $V^{(\sigma)}\doteq (a\in{\cal V}|\sum_{\alpha\in e(a)}\sigma_\alpha>1)$ and $E^{(\sigma)}\doteq (\alpha\in{\cal E}|\sigma_\alpha=1)$.

Each Soft Multi-GM contribution in Eq.~(\ref{eq:LS}) is expressed via a BP solution as follows:
\begin{eqnarray}
&\forall \sigma\in \Sigma_{glp}:\quad & z(\sigma|x^{(bp)})=z(x^{(bp)})
\frac{\prod\limits_{a\in {\cal V}^{(\sigma)}}\mu_a^{(bp)}}{\prod\limits_{\alpha\in {\cal E}^{(\sigma)}} \beta_{\alpha}^{(bp)}(1-\beta_{\alpha}^{(bp)})},\label{eq:r_sigma}\\
&\forall a\in{\cal V}^{(\sigma)}:\quad & \mu_a^{(bp)} \doteq \frac{\sum_{\varsigma_a}f_a(\varsigma_a) \prod_{\alpha}\left((x_{\alpha}^{(bp)})^{\varsigma_{\alpha}}\left(\varsigma_{\alpha}-\beta^{(bp)}_\alpha\right)^{\sigma_{\alpha}}\right)}{\sum_{\varsigma_a}f_a(\varsigma_a) \prod_{\alpha\in e^{(\sigma)}(a)}(x_{\alpha}^{(bp)})^{\varsigma_{\alpha}}},\label{eq:mu_a}\\
&\forall \alpha\in{\cal E}^{(\sigma)}:& \beta_\alpha^{(bp)}\doteq \frac{x_{\alpha_d}x_{\bar{\alpha}_d}}{1+x_{\alpha_d}x_{\bar{\alpha}_d}}, \label{eq:beta_alpha}
\end{eqnarray}
where $e^{(\sigma)}_d(a)$ marks the set of edges of ${\cal E}^{(\sigma)}$ associated with the node $a$ of ${\cal V}^{(\sigma)}$.

\newcommand{\SortNoop}[1]{}


\end{document}